\documentclass{applemlr}

\usepackage{amsmath}
\usepackage{enumerate}
\usepackage{algorithm}
\usepackage{algpseudocode}
\usepackage{amsfonts}
\usepackage{amsthm}
\usepackage{diagbox}
\usepackage{colortbl}
\usepackage{amssymb}
\usepackage{xspace}
\usepackage{wrapfig}
\usepackage{adjustbox}
\usepackage{tabularx}
\usepackage{booktabs}
\usepackage{mathtools}
\usepackage{tikz}
\usepackage{enumitem}
\usepackage{silence}
\usepackage{dsfont}
\usepackage[table]{xcolor}
\usepackage[dvipsnames]{xcolor}
\usepackage{multirow}
\usepackage{makecell}
\usepackage{xfakebold}

\usepackage{amsmath,amsfonts,bm}

\def\eqref#1{equation~\ref{#1}}

\def\1{\bm{1}}

\DeclareMathAlphabet{\mathsfit}{\encodingdefault}{\sfdefault}{m}{sl}
\SetMathAlphabet{\mathsfit}{bold}{\encodingdefault}{\sfdefault}{bx}{n}

\definecolor{textgray}{HTML}{6E6E73}
\usetikzlibrary{positioning, calc}
\usetikzlibrary{decorations.pathmorphing}

\makeatletter
\patchcmd{\wrong@fontshape}{\@gobbletwo}{}{}{}
\makeatother
\numberwithin{equation}{section}
\makeatletter
\AtBeginDocument{
  \urlstyle{sf}
  
}
\makeatother

\definecolor{light}{RGB}{125, 125, 125}
\crefname{tcb@cnt@pbox}{code}{code}
\Crefname{tcb@cnt@pbox}{Code}{Code}
\crefname{assumption}{assumption}{assumption}
\Crefname{assumption}{Assumption}{Assumptions}

\newtcolorbox[auto counter]{pbox}[2][]{
  colback=white,
  title=Code~\thetcbcounter: #2,
  #1,fonttitle=\sffamily,
  fontupper=\sffamily,
  arc=2pt,
  colframe=bgcolor,
  coltitle=fgcolor,
  colbacktitle=bgcolor,
  toptitle=0.25cm,
  bottomtitle=0.125cm
}

\makeatletter
\newcommand\applefootnote[1]{%
  \begingroup
  \renewcommand\thefootnote{}%
  \renewcommand\@makefntext[1]{\noindent##1}%
  \footnote{#1}%
  \addtocounter{footnote}{-1}%
  \endgroup
}
\makeatother

\definecolor{cverbbg}{gray}{0.90}

\usepackage[utf8]{inputenc} %
\usepackage[T1]{fontenc}    %
\usepackage{hyperref}       %
\usepackage{url}            %
\usepackage{booktabs}       %
\usepackage{amsfonts}       %
\usepackage{nicefrac}       %
\usepackage{microtype}      %
\usepackage{xcolor}         %
\usepackage{graphicx}
\usepackage{subcaption}
\usepackage{graphicx} %
\usepackage{amsmath}
\usepackage{amsthm}
\usepackage{amssymb}

\usepackage{xcolor}

\newcommand{\query}{x}
\newcommand{\bx}{\mathbf{x}}
\newcommand{\by}{\mathbf{y}}
\newcommand{\completion}{\by}
\newcommand{\EE}{\mathbb{E}}
\newcommand{\base}{\pi_{\texttt{base}}}
\newcommand{\rlhf}{\pi_{r,\lambda}}
\newcommand{\partition}{Z_{r,\lambda}}

\newcommand{\calD}{\mathcal{D}}
\newcommand{\calN}{\mathcal{N}}
\newcommand{\calL}{\mathcal{L}}
\newcommand{\calV}{\mathcal{V}}

\newcommand{\KL}{\texttt{KL}}
\newcommand{\JS}{\texttt{JD}}

\newcommand{\Cov}{\texttt{Cov}}
\newcommand{\Var}{\texttt{Var}}
\newcommand{\gr}{r^*}
\newcommand{\nr}{\tilde{r}}
\newcommand{\gmodel}{\pi_{\gr,\lambda}}
\newcommand{\nmodel}{\pi_{\nr,\lambda}}

\newtheorem{definition}{Definition}

\usepackage{thm-restate}

\crefname{equation}{}{}
\Crefname{equation}{}{}
\crefname{appendix}{App.}{App.}
\Crefname{appendix}{Appendix}{Appendices}
\crefname{section}{Sec.}{Sec.}
\Crefname{section}{Section}{Sections}
\crefname{table}{Tab.}{Tab.}
\Crefname{table}{Table}{Tables}
\crefname{figure}{Fig.}{Fig.}
\Crefname{figure}{Figure}{Figures}
\crefname{algorithm}{Alg.}{Alg.}
\Crefname{algorithm}{Algorithm}{Algorithms}
\crefname{theorem}{Thm.}{Thm.}
\Crefname{theorem}{Theorem}{Theorems}
\crefname{proposition}{Prop.}{Prop.}
\Crefname{proposition}{Proposition}{Proposition}
\crefname{definition}{Def.}{Def.}
\Crefname{definition}{Definition}{Definition}
\crefname{remark}{Rmk.}{Rmk.}
\Crefname{remark}{Remark}{Remarks}

\usepackage{thmtools, mathtools, thm-restate}
\usepackage{lineno}
\definecolor{darkblue}{rgb}{0, 0, 0.5}
\hypersetup{colorlinks=true, citecolor=darkblue, linkcolor=darkblue, urlcolor=darkblue}

\title{Theoretical Limits of Language Model Alignment}

\author[*]{Lucas Monteiro Paes}
\author[*]{Natalie Mackraz}
\author{Barry-John Theobald}
\author{Federico Danieli}

\affiliation{Apple}

\abstract{
Language model (LM) alignment improves model outputs to reflect human preferences while preserving the capabilities of the base model. The most common alignment approaches are (i) reinforcement learning, which maximizes the expected reward under a KL-divergence constraint, and (ii) best-of-$N$ alignment, which selects the highest-reward output among $N$ independent samples. Despite their widespread use, the fundamental limits of reward improvement under a KL budget remain poorly understood.
We characterize the information-theoretic limits of KL-regularized alignment by deriving the maximum \emph{achievable} expected reward gain for a fixed KL-divergence budget. Our first result provides a closed-form expression for the optimal reward improvement, governed by a Jeffreys divergence term rather than the $\sqrt{\KL}$ used in prior analyses. We further reformulate this expression as a covariance under the base model, yielding a practical estimator that predicts achievable alignment gains from base model samples alone.
We extend our analysis to the proxy reward setting, showing that the gap between ideal and proxy alignment (\emph{reward hacking}) grows with the magnitude of reward error and when the KL penalty factor decreases. 
We then prove that reward ensembling mitigates reward hacking, providing a theoretical justification for this technique used in practice.
Empirically, we compute the KL–reward Pareto frontier for two tasks for LMs, safety and summarization, and show that best-of-$N$ closely approaches the theoretical limit, while PPO and GRPO remain substantially suboptimal.
Our theoretical results shed light on several empirically observed phenomena in the alignment literature and suggest that algorithmic improvements are needed to achieve optimal alignment without high inference costs.
}
\metadata[Correspondence]{\sffamily Lucas Monteiro Paes: \url{lucasmp@apple.com}; \sffamily Natalie Mackraz: \url{natalie_mackraz@apple.com}}
\date{\sffamily\today}

\begin{document}

\maketitle

\section{Introduction}
\label{sec::intro}

Large Language Models (LLMs) perform complex tasks on behalf of users: such as solving mathematical problems \cite{tao2025machine,cobbe2021gsm8k}, moderating online content \cite{kumar2024watch,openai_gpt4_moderation}, writing code \cite{chen2021evaluatinglargelanguagemodels,Anthropic2026ClaudeOpus}, and shopping \cite{webshop,jin2024shopping}.
While these LLMs solve multiple problems without fine-tuning, their output may be toxic, deceiving, or otherwise harmful \cite{ji2023ai,park2024ai,bengio2024managing}.
To ensure that models perform in a safe manner and avoid producing harmful outputs, they usually undergo an \emph{alignment} step to ensure their outputs follow human preferences that ideally reflect ethical values and social norms~\cite{LongHumanPreferences,BuylDiscretion,Buyl2026,huang2025valueswilddiscoveringanalyzing}.

Many approaches for aligning language models have been proposed including, (i) KL-regularized reinforcement learning \cite{bai2022traininghelpfulharmlessassistant,summarizeStiennon}, (ii) direct preference optimization \cite{DPORafailov}, (iii) best-of-$N$ finetuning \cite{touvron2023llama2openfoundation}, among others \cite{MudgalControlled,zhao2023calibrating,gui2024bonbon,verdun2025softbestofnsamplingmodel}.
At their core, these methods aim to solve a KL-regularized reinforcement learning (RL) problem shown in \cref{eq:RLHF}. This approach begins by defining a reward function $r$ which scores model outputs according to how well they align with user preferences.
A base LLM $\base$ is then fine-tuned to maximize the expected reward while penalizing (proportionally to a parameter $\lambda$) the KL-divergence from the base, leading to the aligned model~$\rlhf$~\cite{rlhfChristiano,bai2022constitutionalaiharmlessnessai,bai2022traininghelpfulharmlessassistant,summarizeStiennon}.
Despite the rapid proliferation of alignment algorithms, their theoretical limits remain largely unaddressed.
Progress in alignment is typically measured by achieving higher expected reward for a fixed KL divergence budget~\cite{ScalingLawGao,DPORafailov,MudgalControlled,gui2024bonbon}, where the KL divergence acts as a proxy for how much of the base model's general capabilities we retain after alignment~\cite{paes2026dsodirectsteeringoptimization}.
While alignment methods continuously push the empirical Pareto frontier outward, the true theoretical limit (the maximum achievable reward for a given KL budget) remains generally unknown---see \cref{sec::prelim} for a discussion on the limits of alignment.
Without characterizing this fundamental limit, it is impossible to determine whether current techniques are approaching optimality.%

\begin{figure}[t]
    \centering
    \begin{subfigure}[b]{0.7\textwidth}
        \centering
        \includegraphics[width=\linewidth]{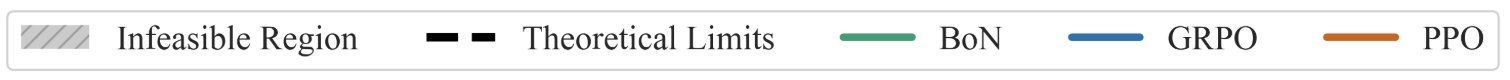}
    \end{subfigure}
    \vfill
    \begin{subfigure}[b]{0.47\textwidth}
        \centering
        \includegraphics[width=\linewidth]{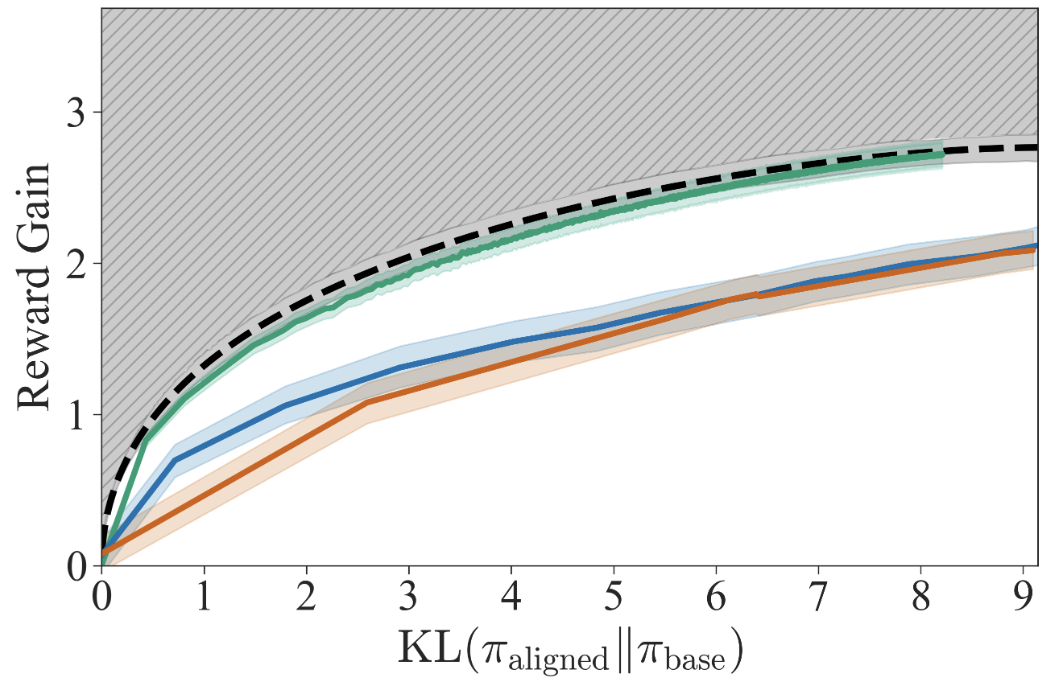} %
        \caption{Zephyr 7B with Beavertails}
    \end{subfigure}
    \hfill %
    \begin{subfigure}[b]{0.47\textwidth}
        \centering
        \includegraphics[width=\linewidth]{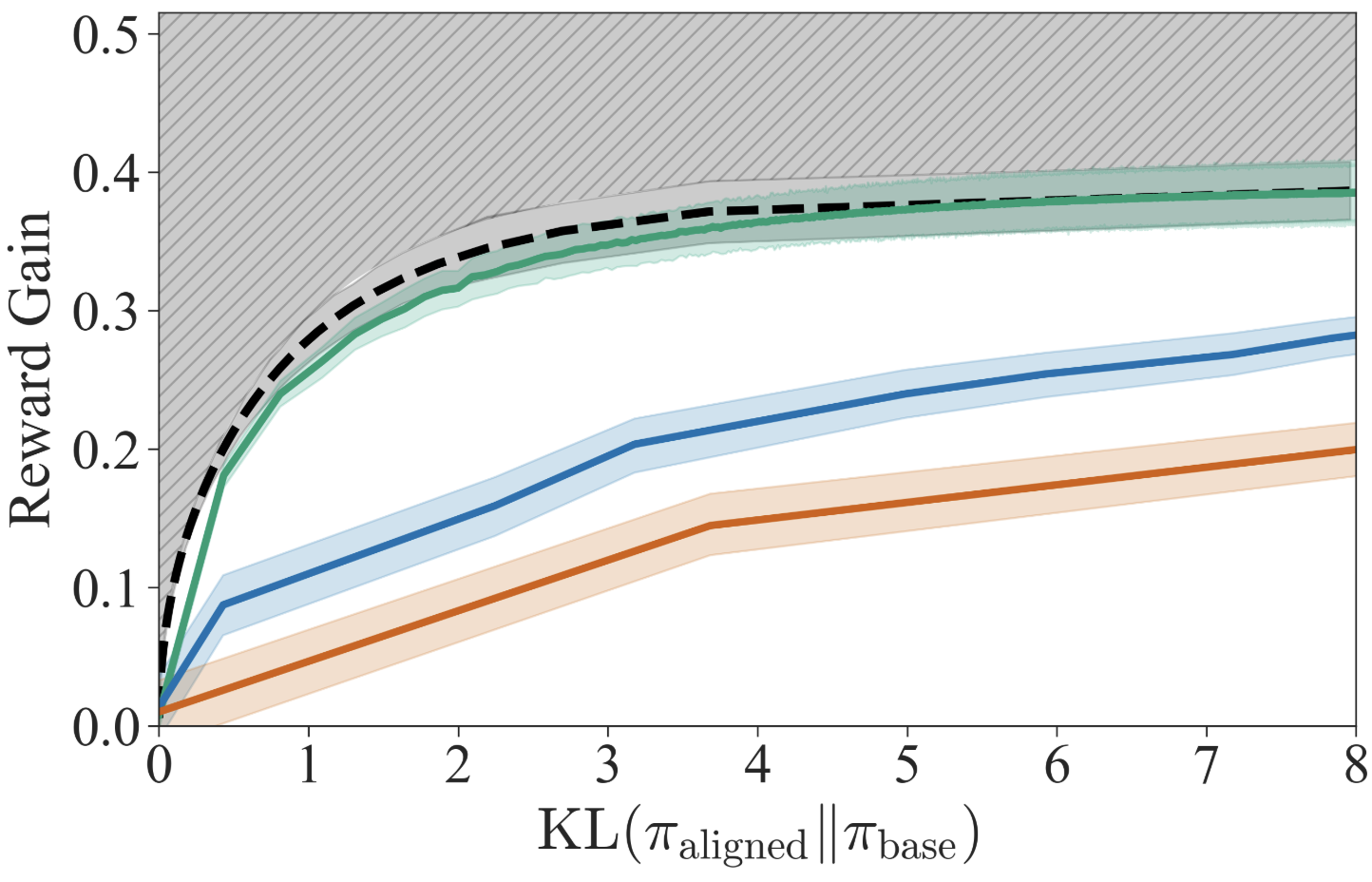}
        \caption{Pythia 1B with TL;DR}
    \end{subfigure}
    \caption{
    \textbf{Fundamental Limits of Alignment.} Reward gain vs.\ KL divergence from $\base$ for best-of-$N$ (BoN), GRPO, PPO, and the theoretical limit computed using \cref{def:mc_estimator} and \cref{def:kl_estimator}. Each PPO/GRPO point corresponds to a checkpoint and each BoN point to a value of $N$.
    In both cases, BoN closely track the performance of the theoretical limits, whereas GRPO and PPO remain sub-optimal. The x-axis goes up to the maximum KL measured by the theoretical limit, which varies across tasks.%
    }
    \label{fig:fundamental_limit}
\end{figure}
\textbf{Fundamental Limits of Alignment.} We explicitly characterize the \emph{fundamental limit of language model alignment}, shown in \cref{fig:fundamental_limit}. First, we show that the reward gain after \emph{optimal alignment} is given by $\lambda \JS(\base|| \rlhf)$, where $\lambda$ is the KL penalty used to restrict the KL divergence (see Equation \cref{eq:RLHF}) and $\JS$ is the Jeffreys divergence, which is given by the sum of the forward and reverse KL divergence, i.e., $\JS(P||Q) \triangleq \KL(P||Q) + \KL(Q||P)$ \cite{jeffreys1998theory}.
This improves upon previous work that provided only a \emph{non-achievable} upper bound characterized by the square root of the KL divergence \cite{bai2022traininghelpfulharmlessassistant,ScalingLawGao,mroueh2025information}. While our characterization is exact, the limit expression depends on the aligned model $\rlhf$, which is available only after \emph{optimal alignment}---undermining its usefulness as a predictive tool.
We resolve this circularity by showing that the Jeffreys divergence term can be rewritten as a covariance under $\base$ alone, thus making it estimable before any post-training is performed.
This gives practitioners a principled way to measure how close their alignment method comes to the optimal trade-off.
\textbf{Proxy Rewards and Ensembling.} The limits of alignment relies on ideal (\emph{gold}) reward models~($\gr$). 
In practice, rewards are learned from human annotations and are noisy approximations, called \emph{proxy rewards} \cite{ScalingLawGao} denoted by~$\nr$. 
We extend our characterization to the proxy-reward setting and show that the gap in reward gain between alignment under the gold reward and under a proxy reward scales with the reward error $\nr - \gr$ and with the inverse of the KL penalty term, i.e., $\lambda^{-1}$, quantifying how reward misspecification degrades alignment and \emph{explaining why aggressive alignment (smaller $\lambda$) using a proxy reward leads to reward hacking as observed by \citet[Figure 1]{ScalingLawGao}}.
We further leverage our framework to justify the effectiveness of \emph{reward ensembling} \cite{eisenstein2024helping,coste2023reward,zhang2024improvingreinforcementlearninghuman} as a practice to mitigate hacking.
We show that, under mild assumptions, averaging the predictions of $n$ reward models causes the reward of the resulting aligned model to converge to the ideal prediction at a rate of $\mathcal{O}(n^{-1/2})$, thus providing a theoretical justification for reward ensembling.

\textbf{Experimental Results.} We evaluate our fundamental limits in two settings: safety alignment using the BeaverTails dataset \cite{ji2023beavertailsimprovedsafetyalignment} and summarization from human feedback using the TL;DR dataset \cite{summarizeStiennon}.
First, we empirically validate the estimators derived from our framework, observing fast convergence when predicting the expected reward gain directly from the base model. 
Second, we benchmark widely used alignment algorithms---specifically PPO~\cite{schulman2017proximal}, GRPO \cite{shao2024deepseekmath}, and best-of-$N$ sampling---against the theoretical Pareto frontier. 
Notably, our results demonstrate that \textbf{best-of-$N$ approaches the fundamental limit}, closely matching our derived optimal reward-KL Pareto frontier --- \cref{fig:fundamental_limit}.

Our main contributions include:
\begin{itemize}
    \item \textbf{Fundamental limit of alignment.} We prove that the optimal reward gain after alignment is exactly $\lambda\,\JS(\base \| \rlhf)$, where $\JS$ is the Jeffreys divergence---\cref{thm:JD_Form}. This disproves the previously believed scaling of $\sqrt{\KL}$ empirically observed in \cite{bai2022traininghelpfulharmlessassistant,ScalingLawGao} and proved to be at least an upper bound for reward~gain~\cite{mroueh2025information}.
    \item \textbf{Fundamental limit estimator.} We show that the optimal reward gain after alignment can be rewritten as a covariance term under the base model alone, making it computable \emph{before alignment}---\cref{thm:characterization_of_reward_gain}. We also derive an estimator of the KL divergence between the aligned model and the base model in \cref{prop:kl_consistency}, allowing us to estimate the reward--KL Pareto frontier.  %
    \item \textbf{Why Overoptimization leads to Reward Hacking.} We characterize how proxy rewards degrade alignment, showing that the post-alignment reward deviates from ideal when reward error grows and the KL penalty decreases---\cref{prop:reward_gap}. This explains why aggressive alignment using a proxy reward leads to reward hacking, as observed by \citet[Figure 1]{ScalingLawGao}
    \item \textbf{Why Reward Ensembling Improves Reward Hacking.} We leverage the covariance form to prove that reward ensembling mitigates reward hacking in \cref{thm:ensemble_efficacy} as previously empirically observed~\cite{coste2023reward,eisenstein2024helping}. 
    \item \textbf{Empirical verification \& Optimality of best-of-$N$.} We validate our estimators and benchmark alignment methods against it (PPO, GRPO, best-of-$N$), demonstrating that best-of-$N$ closely matches our derived theoretical reward---KL Pareto frontier (\cref{fig:fundamental_limit}). Showing why best-of-$N$ was previously observed to be a strong baseline for language model alignment.%
\end{itemize}

\section{Background}
\label{sec::prelim}

\textbf{Language model.} A large language model (LLM) is a conditional probability distribution~$\base$ over sequences of text $\calV$. Given a user query $\bx \in \calV$, the model generates a response $\by$ according to said distribution, $\by\sim \base(\cdot \mid \bx)$. 
In practice, $\base$ is obtained through pretraining and supervised fine-tuning to follow user instructions \cite{Qiu_2020,tie2025surveyposttraininglargelanguage}. 
This training procedure yields a highly capable model often referred to as a \emph{base model}, but model responses may not align with human values and preferences, motivating the need for an alignment procedure to ensure safe model behavior~\cite{bai2022constitutionalaiharmlessnessai,park2024ai,bengio2024managing,bai2022traininghelpfulharmlessassistant}.

\textbf{Reward models.}
Several alignment procedures rely on a \emph{reward model} $r: \calV \times \calV \to \mathbb{R}$ that assigns a scalar score $r(\bx, \by)$ measuring the quality of alignment of response $\by$ to query $\bx$. Ideally, the reward reflects user preferences: if a user prefers $\by^w$ over $\by^l$ in response to $\bx$, then the associated reward should also be larger. 
If a reward model perfectly reflects user preferences it is termed \emph{gold reward} and we denote it by $\gr$.
Reward models are typically learned from preference data 
~\cite{LongHumanPreferences,bai2022traininghelpfulharmlessassistant,DPORafailov,bai2022constitutionalaiharmlessnessai}. However, due to annotation noise and the stochastic nature of training a reward model, the resulting model is an imperfect proxy for the true user preferences, often called a \emph{proxy reward} \cite{ScalingLawGao}, which we denote by $\nr$.
We analyze the difference in models aligned with a gold and a proxy reward in~\cref{sec:impact_of_noise}.

\textbf{KL-regularized alignment.}
At the heart of several alignment techniques
~\cite{LongHumanPreferences,DPORafailov,MudgalControlled}
lies the KL-regularized reinforcement learning problem~\cref{eq:RLHF}.
Specifically, to nudge the base model $\base$ toward user preferences, these approaches seek an aligned policy over the vocabulary $\calV$ that maximizes the expected reward while penalizing deviation from the base model $\base$.
Formally, this amounts to finding the aligned model $\rlhf$ that optimizes the following objective:
\begin{equation}
    \rlhf \triangleq \arg\max_\pi  \frac{1}{|\calD|}\sum_{\query \in \calD} \EE_{\by \sim \pi(\cdot | \bx)}\left[r(\bx, \by)\right] - \lambda \KL(\pi || \base).
    \label{eq:RLHF}
\end{equation}

The KL penalty term in Equation \ref{eq:RLHF} prevents the model from drifting too far from $\base$, with $\lambda>0$ directly controlling the strength of this regularization.
In fact, we show in \cref{prop:reward_gap} that smaller KL penalties are related to more reward hacking.
A fundamental property of Equation \cref{eq:RLHF} is that it admits a closed-form solution given by the following tilted distribution~\cite{AsymptoticsBeirami,salamatian2019mismatchedguesswork}:
\begin{equation}
    \rlhf(\by \mid \bx) = \frac{\base(\by \mid \bx)\, e^{r(\bx, \by)/\lambda}}{\partition(\bx)},\qquad
    \partition(\bx) = \EE_{\by \sim \base(\cdot \mid \bx)}\left[e^{r(\bx,\by)/\lambda}\right],
    \label{eq:post_rlhf}
\end{equation}
where $\partition$ is the \emph{partition function} and normalizes $\base(\by \mid \bx)\, e^{r(\bx, \by)/\lambda}$ into a probability measure $\rlhf$.
We use the result in Equation \cref{eq:post_rlhf} to derive exact characterizations of the reward gain achievable under Equation \cref{eq:RLHF}, and to study how proxy rewards and reward ensembling affect the said gain. %

\textbf{Limits of Alignment \& The Optimality of best-of-$N$.} 
Prior work suggested the optimality of best-of-$N$ \cite{AsymptoticsBeirami,gui2024bonbon,sessa2024bondaligningllmsbestofn}. However, existing justifications rely on structural assumptions on language and reward models. 
\citet{AsymptoticsBeirami} showed that BoN converges to the optimal aligned policy as $N \to \infty$, but assumed linear reward models and memoryless language models. 
Recently, \citet{gui2024bonbon} studied the win-rate--KL Pareto frontier and found that best-of-$N$ approaches it.
However, their analysis focused on win-rate rather than reward maximization which is the optimization objective used in practice, and they also 
assumed continuous language and reward models. 
In contrast, we derive the exact reward--KL Pareto frontier for standard reward maximization without structural assumptions on the language or reward model, giving a direct way to check how close BoN follows the optimum at every value of $N$, and in our experiments BoN closely tracks the entire frontier.

\textbf{Reward Ensembling.}
Reward ensembling has emerged as a practical strategy for mitigating reward hacking. \citet{coste2023reward} show that ensembles of reward models reduce reward hacking.
\citet{eisenstein2024helping} compared different ensembling strategies  showing that some are more effective than others on reducing hacking.
In contrast, \citet{zhang2024improvingreinforcementlearninghuman} and \citet{ahmed2024scalableensemblingmitigatingreward} focused on improving the efficiency of ensembling by building light-weight reward models for the ensemble.
While these works show empirically that reward ensembling mitigates reward hacking in practice, they do not provide a theoretical explanation for the observation, which we establish in \cref{thm:ensemble_efficacy}.

\section{Fundamental Limits of LLM Alignment.}
\label{sec:fundamental_limits}
In this section, we compute the limit of reward gain for a KL-divergence penalty and provide an estimator to compute the limit with convergence guarantees.

\subsection{Information-Theoretic Limits of Alignment}

Previous work has devoted interest to understanding the dependencies of reward improvement and the distance between the base and aligned models (measured as KL divergence).
Specifically, it has been empirically observed that $ \EE_{\rlhf}[r] - \EE_{\base}[r] \approx \sqrt{\sigma \KL(\rlhf || \base) }$ for a constant $\sigma$ \cite{ScalingLawGao}.
Recently, \citet{mroueh2025information} showed that, for a sub-gaussian reward, this empirically observed approximation was actually an upper bound to the reward gain after alignment, i.e., 
\begin{equation}
    \EE_{\rlhf}[r]
   - \EE_{\base}[r] \leq \sqrt{\sigma \KL(\rlhf || \base) }.
   \label{eq:reward_gain_upper}
\end{equation}
The result from \cite{mroueh2025information} in Equation \cref{eq:reward_gain_upper} has three main limitations: (i) it does not provide a closed-form expression for the limits of reward gain, only an upper-bound; (ii) it is not computable (we don't have access to $\rlhf$); (iii) it relies on the assumption that the reward distribution is sub-Gaussian.

\emph{We move away from upper bounds and show a closed-form solution} for the improvement in reward after alignment. Specifically, we prove that the optimal reward improvement is completely characterized by the Jeffreys divergence, i.e., the sum of the forward and reverse KL divergence. 

\begin{restatable}[Information-Theoretic Limits of Alignment]{theorem}{scalingLawThm}~\label{thm:JD_Form}
If $\rlhf$ is the solution of Equation \cref{eq:RLHF} that aligns $\base$ using $r$, then for every fixed query $\bx$ we have
\begin{equation}
    \EE_{\by\sim\rlhf(\cdot \mid \bx)}[r(\by,\bx)]
   - \EE_{\by\sim\base(\cdot \mid \bx)}[r(\by,\bx)] = \lambda \JS(\base || \rlhf).
\end{equation}
\end{restatable}

Theorem \ref{thm:JD_Form} shows that the reward improvement after alignment is fully characterized by the Jeffreys divergence, the sum of the reverse and forward KL divergence, and not only the forward KL divergence as previously believed~\cite{ScalingLawGao,mroueh2025information}.
Notably, our result does not impose any restrictions on the distribution of the reward or language model being used. 
However, for fixed $\lambda$, it is not straightforward to compute the gain in reward because the computation of the Jeffreys divergence requires access to samples from the optimally aligned model $\rlhf$, which \emph{we do not have access to}.

For generality, we measure performance after alignment in our theoretical analysis using a reward $r'$ that \emph{may} be different from the reward we are aligning to $r$. 
This generality is useful to analyze reward gain when we have access to only proxy rewards in~\cref{sec:impact_of_noise}.
To this end, we denote the reward improvement measured with $r'$ after aligning the base model using a reward $r$ as
\begin{equation}
    {\Delta}(r, r') \triangleq \EE_{\rlhf(\cdot \mid \bx)}[r'(\by,\bx)]
   - \EE_{\base(\cdot \mid \bx)}[r'(\by,\bx)].
   \label{eq:reward_gap}
\end{equation}

Next, we provide a characterization of reward gain after alignment that requires access to samples from the base policy $\base$ and scores from the reward model.
\begin{restatable}[Computable Limit of Alignment]{theorem}{covarianceThm}
Let $r'$ be the reward model with respect to which we are measuring performance after alignment using Equation \cref{eq:RLHF}---which may be equal or
 different from the reward model we align the model to ($r$)---then for every fixed query $\bx$, we have
\begin{equation}
\begin{split}
     {\Delta}(r, r') = \EE_{\rlhf}[r'(\by,\bx)]
   - \EE_{\base}[r'(\by,\bx)] =\Cov_{ \base(\cdot \mid \bx)}\left(r', \frac{e^{r(\bx, \by)/\lambda}}{\EE_{\base(\cdot\mid\bx)}\left[e^{r(\bx, \by)/\lambda}\right]}\right).   
\end{split}
\label{eq:reward_gain}
\end{equation}
\label{thm:characterization_of_reward_gain}
\end{restatable}

Theorem \ref{thm:characterization_of_reward_gain} shows that the expected gain in reward after alignment can be written as a function of $\base$ and the reward model, not requiring samples from the optimally aligned model $\rlhf$.
However, the covariance computation requires access to all possible samples $\by \sim \base$, which is infeasible due to the sampling cost and the number of potential responses being combinatorial. We next address this.

\subsection{Estimating the Limits of Alignment from Data}

Having established that the optimal expected reward gain can be exactly expressed as a covariance under the base model in \cref{thm:characterization_of_reward_gain}, we now operationalize this result. 
The critical advantage of the covariance formulation is that it allows us to estimate information-theoretic optimal post-alignment performance without any training, using only generations from the unaligned base model. 
We formalize the estimator for the optimal reward gain ${\Delta}(r, r')$ Equation \cref{eq:reward_gap} next.
\begin{restatable}[Monte Carlo Estimator for ${\Delta}(r, r')$]{definition}{estimatorDefinition}
 \label{def:mc_estimator}
Let $\by^{(1)}, \dots, \by^{(n)}$ be $n$ i.i.d.\ samples from $\base$ for the query $\bx$, $r$ a gold reward model, $r'$ a proxy reward model, and $\lambda > 0$. The Monte Carlo estimator for the expected gain in reward is:
\begin{equation}
    \widehat{\Delta}_n(r, r') = \frac{\widehat{\text{Cov}}}{\widehat{Z}_{r}} = \frac{\widehat{C} - \hat{\mu}_{r'} \widehat{Z}_{r}}{\widehat{Z}_{r}},
\end{equation}
where the components over the $n$ samples are the reward mean $\hat{\mu}_{r'} = \frac{1}{n} \sum_{i=1}^n r'(\by^{(i)})$ (and similarly $\hat{\mu}_{r}$ for $r$), the partition function $\widehat{Z}_{r} = \frac{1}{n} \sum_{i=1}^n \exp(r(\by^{(i)}) / \lambda)$, and the joint term $\widehat{C} = \frac{1}{n} \sum_{i=1}^n \left( r'(\by^{(i)}) \cdot \exp(r(\by^{(i)}) / \lambda) \right)$.
\end{restatable}

While the estimator in \cref{def:mc_estimator} provides a computationally tractable way to estimate the reward gain, its practical utility fundamentally depends on its convergence. 
Next, we formally ensure that as the number of base model samples $n$ grows, our estimate faithfully converges to the theoretical limit.

\begin{restatable}[Convergence of Estimator]{proposition}{estimatorConvergence}
\label{prop:estimator_consistency}
If the expectations $\EE[|r'|]$, $Z_{r}$, and $\mathbb{E}[|r' \cdot \exp(r/\lambda)|]$ are finite, then $\widehat{\Delta}_n(r, r')$ is a consistent estimator for $\EE_{\rlhf(\cdot \mid \bx)}[r'(\by,\bx)] - \EE_{\base(\cdot \mid \bx)}[r'(\by,\bx)]$.
That is, as $n \to \infty$:
$$ \widehat{\Delta}_n(r, r') \xrightarrow{p} {\Delta}(r, r')= \EE_{\rlhf(\cdot \mid \bx)}[r'(\by,\bx)]
   - \EE_{\base(\cdot \mid \bx)}[r'(\by,\bx)], $$
where $\xrightarrow{p}$ denotes convergence in probability.
\end{restatable}
The theoretical Pareto frontier is a trade-off curve, mapping expected reward against base model deviation (KL divergence). Predicting this frontier requires estimating not only the reward gain, but also its corresponding KL divergence for any given penalty $\lambda$. 
Fortunately, the exact same empirical components required to estimate the expected reward are sufficient to construct an estimator for the KL divergence between the optimally aligned model $\rlhf$ and the base $\base$.

\begin{restatable}[Monte Carlo Estimator for KL Divergences]{definition}{klEstimatorDefinition}
\label{def:kl_estimator}
Under the notation as \cref{def:mc_estimator}, we define the KL divergence estimator to be
\begin{align}
    \widehat{\KL}(\rlhf \| \base) &= \frac{\hat{\mu}_r + \widehat{\Delta}_n(r, r)}{\lambda} - \log \widehat{Z}_r. \label{eq:kl_fwd_est}
\end{align}
\end{restatable}

Just as with the reward estimator, we must formally prove that the KL divergence estimator is convergent. 
By mirroring the asymptotic arguments applied to the reward gain, we can similarly guarantee the convergence of the KL divergence estimator.

\begin{restatable}[Convergence of KL Estimators]{proposition}{klConvergence}
\label{prop:kl_consistency}
If the averages $Z_r$ and $\EE[|r \cdot \exp(r/\lambda)|]$ are finite, then $\widehat{\KL}(\rlhf\|\base) \xrightarrow{p} \KL(\rlhf\|\base)$ as $n\to\infty$.
\end{restatable}

Together, the information-theoretic limits derived in this section, alongside the proposed estimators, provide a complete framework to compute the theoretical KL divergence and reward gain post-alignment for any KL penalty value. 
\emph{By evaluating these estimators across a range of penalty parameters $\lambda$, we can systematically estimate the optimal trade-off points and plot the entire Pareto front using only samples from the base model.}

Our results allows practitioners to directly compare the empirical performance of various alignment algorithms against our derived theoretical Pareto frontier. 
By establishing this exact theoretically achievable limit, we can now quantify how close current methods are to information-theoretic optimality. 
Ultimately, this framework provides a definitive benchmark to determine whether substantial headroom for algorithmic improvement remains, or if further refinements to existing alignment methods will inevitably lead to diminishing returns.

The limits established in this section apply to the setting where the reward optimized during alignment is the same reward used for evaluation or we have access to the gold reward that will be used at the time of evaluation. 
In practice, however, we lack direct access to a ``gold'' reward model $\gr$. 
Instead, alignment relies on proxy reward models $\nr$ trained from finite and noisy human preference data. Next, we connect our results to proxy rewards used in practice.

\section{Proxy Rewards and a Justification for Reward Ensembling}
\label{sec:impact_of_noise}
In this section, we leverage our covariance characterization (\cref{thm:characterization_of_reward_gain}) to (i) quantify how aligning to proxy rewards degrades alignment performance on the gold reward (reward hacking)~\cite{amodei2016concreteproblemsaisafety,ScalingLawGao} and to (ii) theoretically explain why reward ensembling works as a mechanism to mitigate reward hacking.

\subsection{The Impact of Proxy Rewards in Reward Hacking}

To analyze the impact of the proxy reward on alignment, we first formalize the discrepancy between the ideal evaluation metric (gold reward) and the imperfect training signal (proxy reward).

\begin{definition}[Preference Residual] 
A reward model $\gr$ is a \emph{gold reward} if it perfectly quantifies the preferences of a user. 
A proxy reward $\nr$ is an approximation for $\gr$ that we usually have access to (e.g., model trained from pairwise preference).
We denote the \textbf{preference residual} between the gold and proxy rewards by $\calN(\bx, \by) \triangleq \nr(\bx, \by)  - \gr(\bx, \by)$.
\end{definition}

The penalty incurred by aligning to an imperfect proxy is quantified by the performance gap between a model aligned to the gold reward and one aligned to the proxy. 
\Cref{prop:reward_gap} computes the difference in expected true reward between these two policies.%

\begin{restatable}[The Impact of Proxy Rewards]{proposition}{ProxyRewardImpact}
\label{prop:reward_gap}
Let $\gr$ be the gold reward, $\nr$ the proxy reward, and $\calN(\bx,\by)$ the preference residual.
Denote by $\gmodel$ the aligned model using the gold reward and, $\nmodel$ the aligned model using the proxy reward --- both optimally aligned as per \cref{thm:characterization_of_reward_gain}. 
The difference in reward 
is given by:
{ \small \begin{align} 
    \EE_{\gmodel(\cdot\mid\bx)}\big[\gr(\bx,\completion)\big] - \EE_{\nmodel(\cdot\mid\bx)}\big[\gr(\bx,\completion)\big]
    &= \Cov\left( r^*(\by),\frac{ e^{r^*(\by)/\lambda}}{\EE\big[e^{\gr(\completion)/\lambda}\big]} \left( 1 - \frac{e^{\calN(\bx, \completion)/\lambda}}{\EE_{\pi^*}\big[e^{\calN(\bx, \completion)/\lambda}\big]}\right)\right).
    \label{eq:reward_gap_result}
\end{align}}
\end{restatable}

\begin{figure}[t]
    \centering
    \begin{subfigure}[b]{0.6\textwidth}
        \centering
        \includegraphics[width=\linewidth]{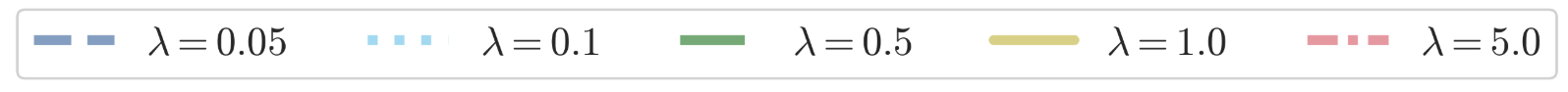}
    \end{subfigure}
    \vfill
    \begin{subfigure}[b]{0.47\textwidth}
        \centering
        \includegraphics[width=\linewidth]{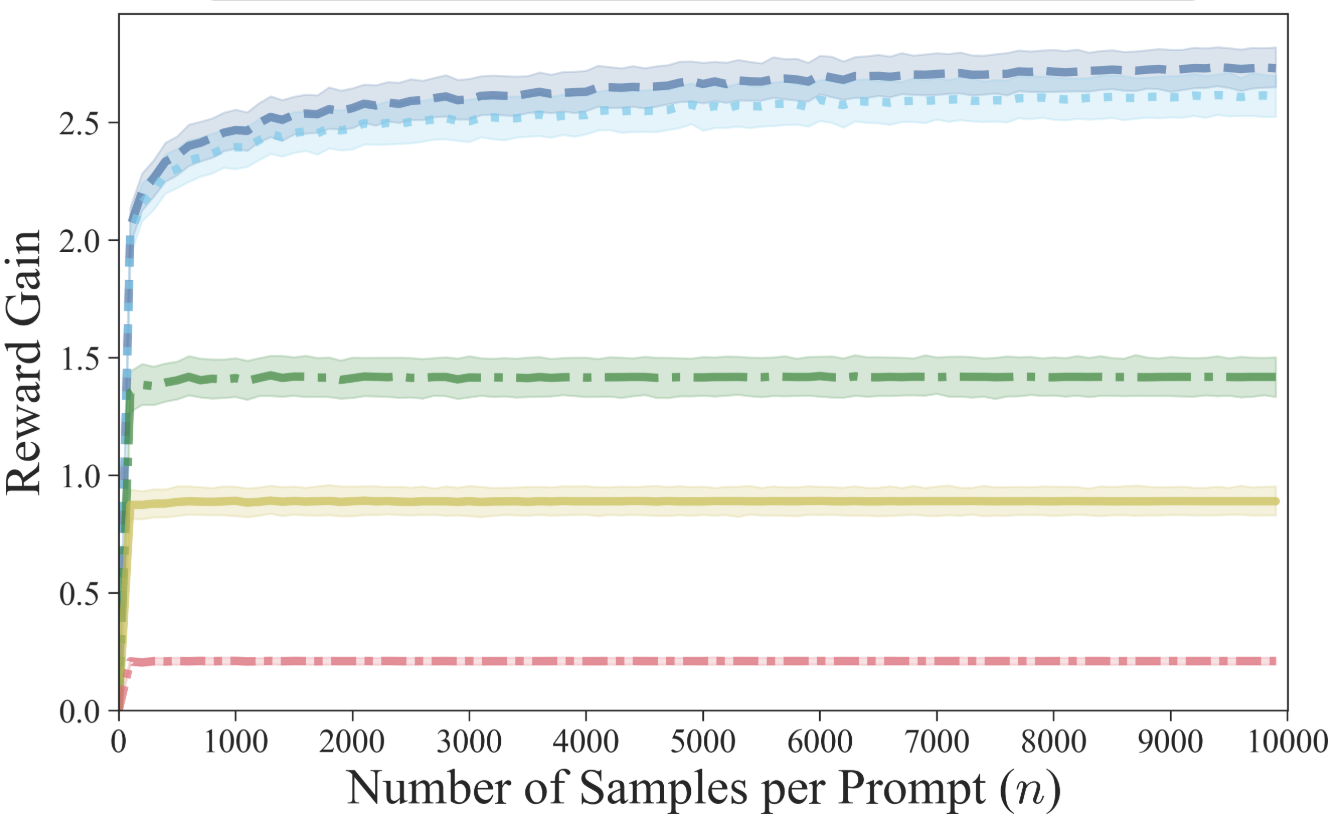}
        \caption{Zephyr 7B with BeaverTails}
    \end{subfigure}
    \hfill
    \begin{subfigure}[b]{0.48\textwidth}
        \centering
        \includegraphics[width=\linewidth]{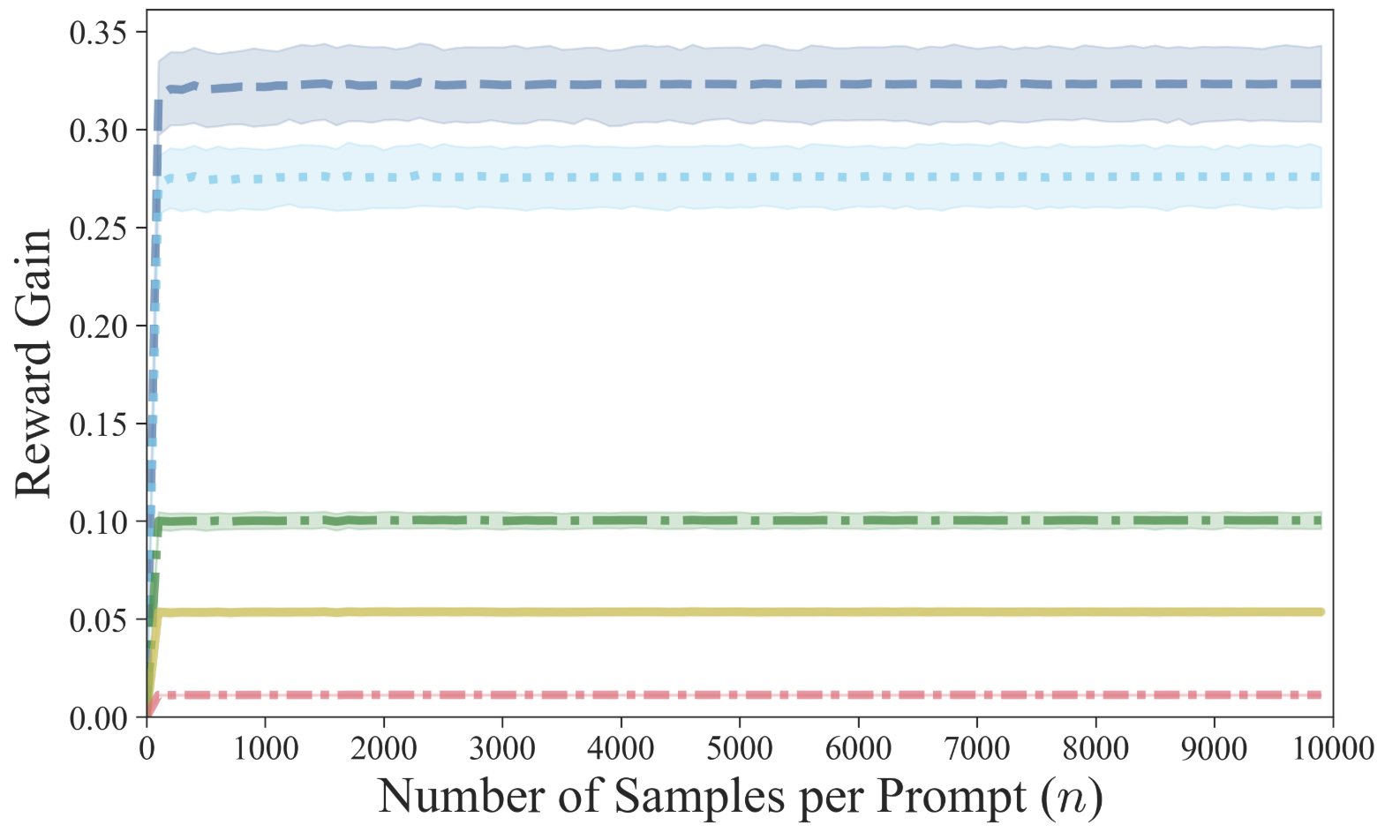}
        \caption{Pythia 1B with TL;DR}
    \end{subfigure}
    \caption{
    \textbf{Convergence of the reward-gain estimator.} Estimated reward gain ($\widehat{\Delta}_n(r,r)$ --- \cref{def:mc_estimator}) in the y-axis as a function of the number of base model samples used to estimate it ($n$) in the x-axis. We vary the KL penalty across $\lambda \in \{0.05, 0.1, 0.5, 1.0, 5.0\}$. Shaded bands show 95\% confidence intervals using bootstrap from Seaborn \cite{Waskom2021}.
    }
    \label{fig:estimation_gain}
    \vspace{-1em}
\end{figure}

\Cref{prop:reward_gap} characterizes \emph{reward hacking}. 
The reward degradation is not merely proportional to the preference residual $\mathcal{N}$, it is actually governed by the exponential term $e^{\mathcal{N}(\bx, \by)/\lambda}$.
This reveals a critical vulnerability in KL-regularized alignment: the performance degradation is governed by the preference residual.
Moreover, as the penalty $\lambda$ decreases (allowing the model to drift further from the base policy), the impact of the proxy reward increases exponentially. 
This explicitly demonstrates why aggressive alignment against proxy rewards often leads to severe reward hacking as previously observed empirically by \citet[Figure 1]{ScalingLawGao}.

\subsection{Why Reward Ensembling Mitigates Reward Hacking}

Given the exponential reward degradation caused by preference residuals from proxy rewards, practitioners require immediate strategies to counter reward hacking. 
A common heuristic consists of using an ensemble of reward models rather than a single proxy reward \cite{coste2023reward,eisenstein2024helping}.
While ensembling has been shown to be effective in practice, it lacks formal theoretical guarantees within the context of KL-regularized alignment.

To prove that ensembling mitigates reward hacking, we first introduce a bound that connects the difference in expected reward between two aligned models directly to the distance between their underlying reward functions under the base model. 
This result implies that if a reward model is close to the gold reward ---in terms of average preference residuals--- then the reward gain from alignment generalizes to the gold rewards.

\begin{restatable}[A Deterministic Lipschitz Bound]{proposition}{lipsBound}
Denote the preference residual between a gold $\gr$ and a proxy $\nr$ reward by $\calN$.
If the gold reward and the proxy are both bounded, then there exists a constant $C > 0$ such that,
\begin{equation}
    \big|\EE_{\gmodel}[\gr] - \EE_{\nmodel}[\gr]\big| \leq \textsc{C}_{\lambda} \times \Var(\gr)^{1/2} \|\gr - \nr\|_{L^2(\base)} = \textsc{C}_{\lambda} \times \Var(\gr)^{1/2} \|\calN\|_{L^2(\base)}.
    \label{eq:det_lipschitz}
\end{equation}
\label{prop:lipschitz_bound}
\end{restatable}
\Cref{prop:lipschitz_bound} guarantees that if we can construct a proxy reward that is close to the true reward in terms of the $\ell_2$ norm of the preference residual, the \emph{post-alignment performance of the model aligned to the proxy reward will generalize to the gold reward}.
Now, we can apply the bound in Equation \cref{eq:det_lipschitz} to an ensemble average to derive its exact rate of convergence, demonstrating \emph{why} reward ensembling mitigates reward hacking and complementing previous work that have empirically observed this result~\cite{coste2023reward,ahmed2024scalableensemblingmitigatingreward,eisenstein2024helping,zhang2024improvingreinforcementlearninghuman}.

\begin{restatable}[Rate of Convergence for Reward Ensemble]{theorem}{ensembleRate}
\label{thm:ensemble_efficacy}
 Let $r_1, ..., r_n$ be approximations of the perfect reward $\gr$ sampled i.i.d.\ from a distribution $P$ (e.g., different reward training runs).
Define $R_n$ as the mean of the proxy rewards, i.e., $R_n(\bx, \by) \triangleq \frac{1}{n}\sum_{i = 1}^{n} r_i(\bx, \by)$. If $\EE_{r_i \sim P}[r_i(\bx, \by)] = \gr(\bx, \by)$ holds and each $i \in [n]$ and $\gr$ and $r_i$ are bounded, then 
\begin{equation}
    \EE_{\pi(R_n, \lambda)} \left[ \gr \right] \rightarrow \EE_{\gmodel} \left[ \gr \right], \ \ \text{ at rate of } O\left(n^{-1/2}\right) \text{ in expectation over the $P$.}
\end{equation}
\end{restatable}

\section{Experiments}
\label{sec:experiments}
In this section, we present experiments demonstrating that (i) the limit estimators proposed in \cref{prop:estimator_consistency} and \cref{prop:kl_consistency} quickly converge, (ii) best-of-$N$ is close to the theoretical limit, i.e., the KL---Reward Pareto frontier, and (iii) both GRPO and PPO are suboptimal in terms of the KL vs. Reward gain.

\paragraph{Dataset \& Tasks.} 
We evaluate our theoretical limits on two alignment tasks: summarization and safety.
For summarization, we use the \texttt{Reddit TL;DR} dataset \cite{SummFromHuman}, which provides Reddit posts paired with two candidate summaries and human preference labels; the task is to generate post summaries (TL;DRs) aligned with annotator preferences --- during RL we use the preprocessed variant of the dataset \cite{trl-lib-tldr}.
For safety, we use \texttt{BeaverTails} \cite{ji2023beavertailsimprovedsafetyalignment}, which consists of prompts designed to elicit harmful completions; the task is to align the model to produce safe responses.

\begin{figure}[t]
    \centering
    \begin{subfigure}[b]{0.6\textwidth}
        \centering
        \includegraphics[width=\linewidth]{estimation_limit/kl_legend_no_title_with_lambda.png}
    \end{subfigure}
    \vfill
    \begin{subfigure}[b]{0.47\textwidth}
        \centering
        \includegraphics[width=\linewidth]{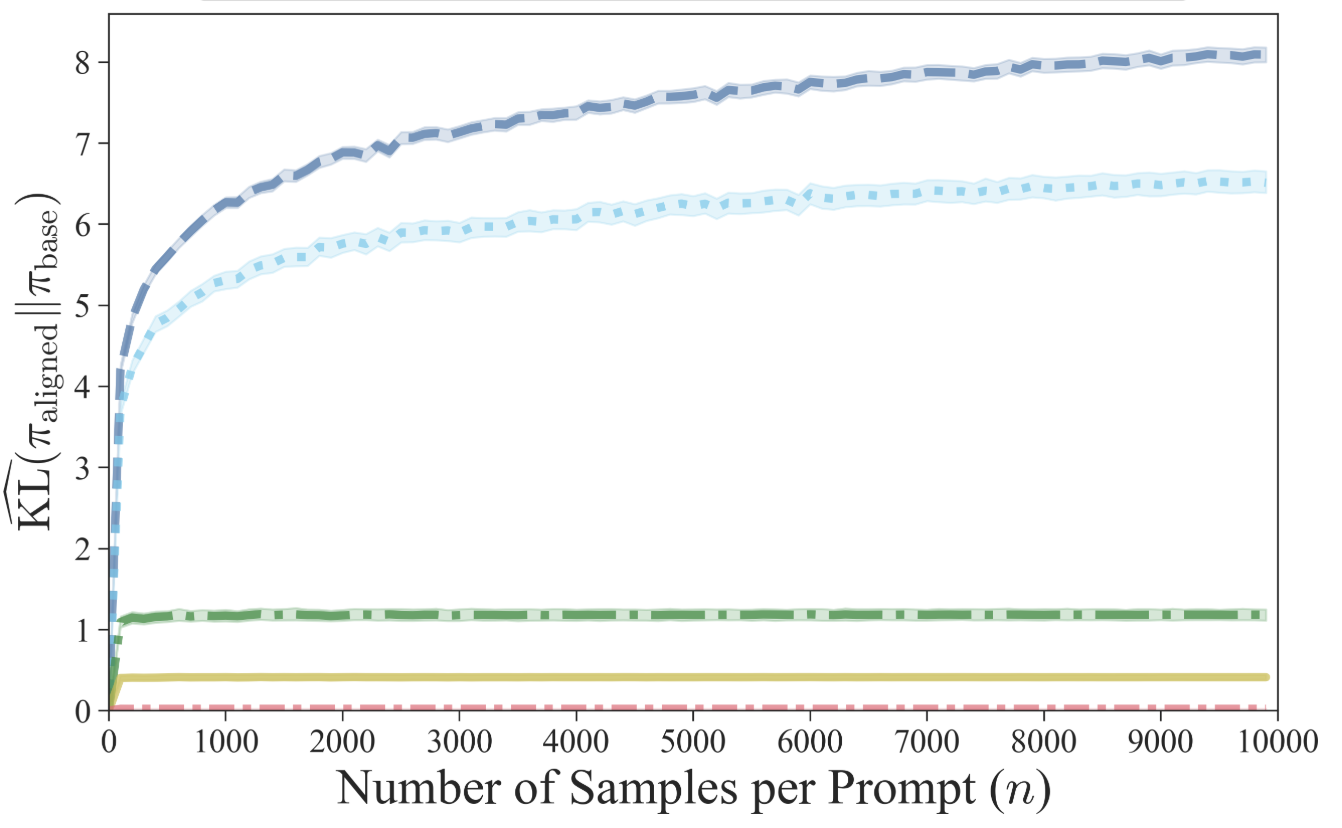}
        \caption{Zephyr 7b with BeaverTails}
    \end{subfigure}
    \hfill
    \begin{subfigure}[b]{0.48\textwidth}
        \centering
        \includegraphics[width=\linewidth]{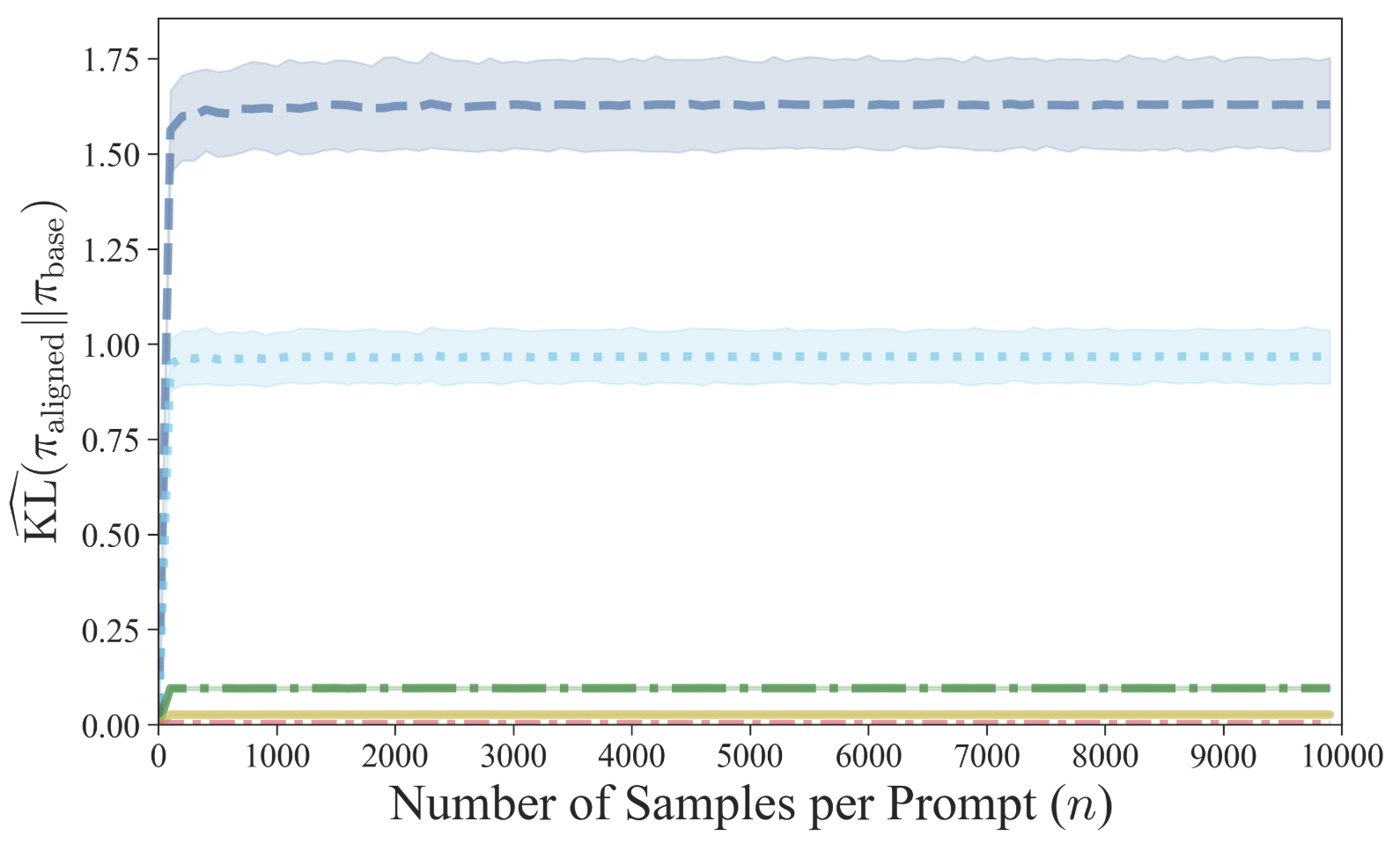}
        \caption{Pythia 1B with TL;DR}
    \end{subfigure}
    \caption{
    \textbf{Convergence of the KL Estimator.} Estimated KL divergence between the aligned and the base model ($\widehat{\KL}(\rlhf \| \base)$ --- \cref{def:kl_estimator}) in the y-axis as a function of the number of samples used to estimate it in the x-axis. The KL penalty varies across $\lambda \in \{0.05, 0.1, 0.5, 1.0, 5.0\}$. Confidence (95\%) intervals computed using bootstrap~\cite{Waskom2021}.
    }
    \label{fig:estimation_kl}
    \vspace{-1em}
\end{figure}

\paragraph{Base \& Reward Models.}
For summarization (\texttt{Reddit TL;DR}), we start from the \emph{base model} \texttt{Pythia-1B}~\cite{biderman2023pythiasuiteanalyzinglarge} supervised fine-tuned on \texttt{Reddit TL;DR}~\cite{trl_pythia1b_tldr_sft}, and use a \texttt{Pythia-1B}-based \emph{reward model} trained on pairwise TL;DR preferences~\cite{cleanrl_pythia1b_reward_tldr}, reproducing the setup of~\citet{huang2024nimplementationdetailsrlhf}.
For the safety task (\texttt{BeaverTails}), we start from \texttt{Zephyr-7B}~\cite{tunstall2023zephyrdirectdistillationlm}, a \emph{base model} obtained by supervised fine-tuning \texttt{Mistral-7B}~\cite{jiang2023mistral7b} on the UltraChat dataset~\cite{ding-etal-2023-enhancing}, and train a \texttt{Zephyr-7B}-based \emph{reward model} on pairwise safety preferences from \texttt{PKU-SafeRLHF}~\cite{ji2025pkusaferlhfmultilevelsafetyalignment} using the TRL reward trainer~\cite{vonwerra2020trl}.

\paragraph{Alignment Algorithms.} 
We compare the theoretical limit against three alignment algorithms: {GRPO}~\cite{shao2024deepseekmath} and {PPO}~\cite{schulman2017proximal}, two KL-regularized RL algorithms, and {best-of-$N$} (BoN), an inference-time method that has proven to be a strong baseline~\cite{MudgalControlled,DPORafailov,AsymptoticsBeirami,gui2024bonbon}. 
We do not consider DPO~\cite{DPORafailov}, as it does not use an explicit reward model.
Since both {GRPO} and {PPO} are sensitive to hyperparameter choices, we performed an extensive sweep (\cref{apx:hparam_sweeps}) and report the best run, measured by the KL-reward trade-off, in the main paper.
For best-of-$N$, we vary $N$ between $1$ and $10\_000$~($10k$) and compute its KL divergence using the estimator of proposed by \citet{beirami2025theoretical}.

\subsection{Experimental Results}

\paragraph{Estimator convergence.}
Before benchmarking algorithms against the fundamental limit, we analyze how quickly our estimators converge to their asymptotic values. \cref{fig:estimation_gain} reports the reward-gain estimator $\widehat{\Delta}_n(r,r)$ (\cref{def:mc_estimator}) and \cref{fig:estimation_kl} reports the KL-divergence estimator $\widehat{\KL}(\rlhf|\base)$ (\cref{def:kl_estimator}), as a function of the number of generations per query and across varying penalties ($\lambda \in \{0.1, 0.5, 1, 5\}$). 

In \cref{fig:estimation_gain}, the reward-gain estimator converges as predicted by \cref{prop:estimator_consistency}, typically requiring fewer than one thousand samples per input query. 
\cref{fig:estimation_kl} shows analogous convergence behavior for the KL estimator, as predicted by \cref{prop:kl_consistency}. 
These results confirm that reward gain and KL divergence are computable from base model samples alone, without access to the aligned policy.
In both estimations, convergence is faster for larger $\lambda$ and slower for small $\lambda$, where the exponential tilt $e^{r/\lambda}$ amplifies the variance of high-reward samples.
Together, these results validate the estimators used for both \cref{fig:fundamental_limit} and \cref{fig:estimation_kl}, and show that the full theoretical Pareto frontier can be reliably estimated from a few thousand base model samples, before any alignment training.

\paragraph{KL--reward Pareto frontier.}
\cref{fig:fundamental_limit} compares the theoretical KL-reward frontier computed using $\widehat{\Delta}_n$ from \cref{def:mc_estimator} and $\widehat{\KL}(\rlhf|\base)$ from \cref{def:kl_estimator} against the empirical KL-reward gain achieved by BoN, GRPO, and PPO.
The theoretical frontier is computed using $10k$ samples per query from base-model--- \cref{fig:estimation_gain} and \cref{fig:estimation_kl} show that $10k$ samples is enough for an accurate estimation, but as few as $1k$ samples also suffices as we discuss in \cref{apx:ParetoFront}.
We vary $\lambda$ between $0.01$ and $10$ to trace the full frontier.
The black curve corresponds to the information-theoretic \emph{achievable} limit on the reward gain at each KL budget (\cref{thm:JD_Form}). 
The gray region above it corresponds to \emph{unachievable} reward gains as these would violate the fundamental limit. 
Points on or below the curve are achievable, and the gap between an algorithm's empirical point and the frontier quantifies how much reward is being left on the table.

\paragraph{best-of-$N$ is almost Pareto optimal.} 
\cref{fig:fundamental_limit} shows that BoN's KL vs. reward gain points lies on or near the Pareto curve at \emph{every} KL value, equivalently, every number of samples $N$.
Prior work established that BoN approaches the information-theoretic limit asymptotically as $N \to \infty$ (equivalently, $\KL \to \infty$)~\citep{AsymptoticsBeirami}, our experiments suggest that BoN tracks the KL-optimal frontier across the \emph{entire trade-off curve}. 
Therefore, BoN is not merely a strong baseline that approaches optimality when $N \to \infty$, it is close to the best any alignment algorithm can achieve at every KL budget.

\paragraph{GRPO and PPO are suboptimal.} 
In \cref{fig:fundamental_limit}, we observe that both GRPO and PPO remain substantially below the theoretical frontier at every KL budget tested, recovering only a fraction of the reward gain achieved by BoN and the information-theoretic limit. 
This points to a structural limitation of current RL methods, which appears to operate far from the KL-optimal solution to Equation~\cref{eq:RLHF}.

\section{Final Discussion}
\label{sec:conclusion}

\textbf{Conclusion.} 
We characterized the information-theoretic limits of language model alignment in terms of KL divergence vs. reward gain.
First, we derived the optimal reward improvement after alignment showing that the gain in expected reward is governed by the Jeffrey's divergence $\lambda\,\JS(\rlhf \Vert \base)$, improving upon the previously conjectured $\sqrt{\KL(\rlhf \Vert \base )}$ scaling which was empirically observed \cite{bai2022traininghelpfulharmlessassistant,ScalingLawGao} and proven to be an upper bound for optimal alignment by \citet{mroueh2025information}.

Second, we showed that the Jeffreys divergence term can be rewritten as a covariance under the base model, yielding an exact characterization of the optimal reward gain that does not require access to the aligned model. 
Building on this, we proposed convergent estimators that recover the full KL--reward Pareto frontier from base model samples alone. 
This enables practitioners to forecast achievable reward gains and benchmark alignment algorithms against an \emph{achievable limit}.
Empirically, we found that best-of-$N$ sampling closely tracks the theoretical Pareto frontier, while PPO and GRPO are suboptimal, pointing to significant room for algorithmic improvement in gradient-based alignment.

Third, we leveraged our covariance characterization to analyze the gap between optimally aligned LMs using the gold reward and under proxy rewards. 
We proved that this gap is governed by the inverse KL penalty $\lambda^{-1}$ and the mean square error between proxy and gold rewards.
The dependency on the KL penalty explains why aggressive alignment with proxy rewards leads to reward hacking, which was previously observed in \cite{ScalingLawGao}. 
Building on the same characterization, we show that reward ensembling provably mitigates reward hacking, with convergence to ideal-reward alignment at a rate of $\mathcal{O}(n^{-1/2})$, providing a theoretical guarantee for this widely used practice.

\textbf{Implications and Future Work.} 
On the practical side, our estimators for the KL vs.\ reward gain Pareto frontier open the door to principled hyperparameter selection and algorithm evaluation: practitioners can rely on the Pareto frontier to identify the KL budget required for a target reward level, and to compare the performance of alignment algorithms in absolute terms.

On the modeling side, our analysis on the use of proxy rewards exposes a structural limitation of the dominant reward-modeling paradigm. 
Currently, rewards are trained under the Bradley-Terry (BT) model. 
This approach captures only the \emph{relative preference} of responses, but not \emph{how much} one is preferred over another. 
As a consequence, the preference residual $\calN$ between learned proxy and gold reward can grow unboundedly, even with a perfectly minimized BT loss --- see \cref{apx:leaveBTbehind} for a more detailed justification.
However, \cref{prop:reward_gap} shows that the preference residual is precisely the culprit behind the degradation of alignment quality. 
These findings suggest that a \emph{new reward training paradigm}, capable of directly constraining the preference residual (for instance, by incorporating preference \emph{magnitudes}) could mitigate reward hacking at its source, rather than relying on post-hoc corrections such as ensembling. 
We believe that designing such feedback protocols, and understanding how they interact with the fundamental limits derived here, are important directions for future work.

\textbf{Limitations.}
Our work has limitations we hope future work will address. First, our results characterize the optimal solution to the KL-regularized objective in Equation~\cref{eq:RLHF}; extending the analysis to alternative objectives (e.g., $f$-divergence regularization, win-rate maximization, per-token KL divergence) is an interesting direction.
Second, our proposed estimators converge slowly for smaller $\lambda$ (large $\KL$); future work could focus on designing better estimators for this regime.
Finally, our analysis of reward hacking (\cref{prop:reward_gap}) and ensembling (\cref{thm:ensemble_efficacy}) assumes proxy rewards are unbiased estimates of the gold reward on average across proxy reward samples; characterizing the impact of systematic biases in proxy reward design is left for future work.

\applefootnote{ \textcolor{textgray}{\sffamily Apple and the Apple logo are trademarks of Apple Inc., registered in the U.S. and other countries and regions.}}

\newpage
\bibliography{bibliography}
\bibliographystyle{plainnat}

\clearpage
\appendix
\begin{center}
{\LARGE \bf Supplementary Material: \\ Theoretical Limits of Language Model Alignment \par}
\vspace{1cm}
\end{center}

\section{Preliminaries and Notation}

\paragraph{Preliminaries.} In the supplementary material we provide the following information:

\begin{itemize}
    \item \cref{apx:theory} provides the proofs for all theoretical results presented in the paper.
    \item \cref{apx:leaveBTbehind} discusses the limitations of the Bradley-Terry (BT) reward-modeling paradigm and provides a path to improve reward hacking with a new approach for reward modeling that accounts for preference residuals.
    \item \cref{apx:ParetoFront} presents additional empirical results that validate the convergence of the estimated Reward-KL Pareto frontier.
    \item \cref{apx:hparam_sweeps} provides the full details for the extensive hyper-parameter sweeps performed for the GRPO and PPO baseline experiments.
    \item \cref{apx:assets} details the licenses for all public datasets, model checkpoints, and software libraries utilized in our evaluation.
\end{itemize}

\paragraph{Notation Table.} We show the notation used through the paper in \cref{tab:notation}.

\begin{table}[htb!]
\centering
\renewcommand{\arraystretch}{1.1}
\begin{tabular}{ll}
\toprule
\textbf{Symbol} & \textbf{Meaning} \\
\midrule
$\calV$ & vocabulary of the language model \\
$\bx$ & user prompt \\
$\by$ & model response \\
$\base$ & base language model (before alignment) \\ 
$\lambda$ & KL penalty parameter in Equation \cref{eq:RLHF}\\
$\rlhf$ & optimally aligned language model policy for reward $r$ and penalty $\lambda$ \\
$r$ & general reward model function \\
$\gr$ & ideal (gold) reward model \\
$\nr$ & proxy reward model \\
$\KL(P || Q)$ & Kullback-Leibler divergence between distributions $P$ and $Q$ \\
$\JS(P || Q)$ & Jeffreys divergence, i.e., $\JS(P || Q) \triangleq \KL(P || Q) + \KL(Q || P)$\\
$\Delta(r, r')$ & expected reward gain measured with $r'$ after aligning with $r$ \\
$\widehat{\Delta}_n$ & Monte Carlo estimator for the expected reward gain \\
$\calN(\bx, \by)$ & preference residual between gold and proxy rewards, i.e.,  $ \calN(\bx, \by) \triangleq \nr - \gr$ \\
$\partition(\bx)$ & partition function for the tilted distribution \\
$N$ & number of samples per query generated for best-of-$N$ \\
$n$ & number of base model samples per query used for Monte Carlo estimation \\
\bottomrule
\end{tabular}
\vspace{0.2cm}
\caption{Notation used throughout the main paper and appendix.}
\label{tab:notation}
\end{table}
\clearpage

\section{Theoretical Results and Proofs}
\label{apx:theory}
\subsection{Proof of Theoretical results in Section~\ref{sec:fundamental_limits} }
\label{apx:proof_sec3}

\scalingLawThm*
\begin{proof}
As per \cref{eq:post_rlhf}, the following relationship holds between aligned and base models:
\begin{equation}
\frac{\pi(r,\lambda)(\by\mid \bx)}{\base(\by\mid \bx)}
= \frac{e^{r(\bx,\by)/\lambda}}{\partition(\bx)}.
\end{equation}
Taking logarithms both sides gives
\begin{equation}
\begin{split}
&\log \pi(r,\lambda)(\by\mid \bx) - \log \base(\by\mid \bx) = \frac{r(\bx,\by)}{\lambda} - \log \partition(\bx) \\
\Longrightarrow\quad& r(\bx,\by) = \lambda \left(\log\frac{\pi(r,\lambda)(\by\mid \bx)}{\base(\by\mid \bx)} + \log \partition(\bx)\right).
\end{split}
\label{eqn:thm1_step1}
\end{equation}
Taking the expectation of \cref{eqn:thm1_step1} under $\pi(r,\lambda)$ gives
\begin{equation}
    \EE_{\by\sim\pi(r,\lambda)}[r(\bx,\by)]
= \lambda \KL\left(\pi(r,\lambda)|| \base\right)
+ \lambda \log \partition(\bx),
\end{equation}
while taking the expectation of \cref{eqn:thm1_step1} under $\base$ similarly yields
\begin{equation}
\EE_{\by\sim\pi_{\text{base}}}[r(\bx,\by)]=
 -\,\lambda\KL\big(\base || \pi(r,\lambda)\big)
+ \lambda \log \partition(\bx).
\end{equation}
Finally, subtracting the two expectations we get
\begin{equation}
\begin{split}
\EE_{\by\sim\pi(r,\lambda)}[r(\bx,\by)] - \EE_{\by\sim\base}[r(\bx,\by)] \nonumber & = \lambda \left( \KL(\pi(r,\lambda) || \base) + \KL(\base || \pi(r,\lambda))\right) \\
&= {\lambda} \JS(\pi(r,\lambda) || \base),
\end{split}
\end{equation}
with $\JS$ denoting Jeffreys divergence, which completes the proof.
\end{proof}

\covarianceThm*
\begin{proof}
First, note that using \cref{eq:post_rlhf} we can write the expectation under the aligned model as
\begin{equation}
\EE_{\by\sim\rlhf} \left[ r'(\bx,\by) \right] =\sum_{\by \in \calV} \frac{\base(\by \mid \bx) e^{r(\bx, \by)/\lambda}}{\EE_{\by\sim\base}\left[e^{r(\bx, \by)/\lambda}\right]} \, r'(\bx, \by)
=\frac{\EE_{\by\sim\base}\left[ e^{r(\bx, \by)/\lambda} \, r'(\bx, \by)\right]}{\EE_{\by\sim\base}\left[e^{r(\bx, \by)/\lambda}\right]}.
\label{eq:base_tilt_change}
\end{equation}
Plugging this into \cref{eq:reward_gain}, we have
\begin{equation}
\begin{split}
    &\EE_{\by\sim\rlhf} \left[ r'(\bx,\by) \right] - \EE_{\by\sim\base} \left[ r'(\bx,\by) \right] \\
    =& \frac{\EE_{\by\sim\base}\left[ e^{r(\bx, \by)/\lambda} \, r'(\bx, \by)\right]}{\EE_{\by\sim\base}\left[e^{r(\bx, \by)/\lambda}\right]} - \EE_{\by\sim\base} \left[ r'(\bx,\by) \right] \\
    =& \frac{\EE_{\by\sim\base}\left[ e^{r(\bx, \by)/\lambda} \, r'(\bx, \by)\right]  -  \EE_{\by\sim\base}\left[e^{r(\bx, \by)/\lambda}\right]\EE_{\by\sim\base} \left[ r'(\bx,\by) \right]}{\EE_{\by\sim\base}\left[e^{r(\bx, \by)/\lambda}\right]}
\end{split}
\label{eq:almost_cov}
\end{equation}
Note that the numerator in \cref{eq:almost_cov} is in the form $\EE[XY] - \EE[X]\EE[Y]$, which is exactly the covariance between $X$ and $Y$. This allows us to rewrite \cref{eq:almost_cov} as
\begin{equation}
\begin{split}
    \EE_{\by\sim\rlhf} \left[ r'(\bx,\by) \right] - \EE_{\by\sim\base} \left[ r'(\bx,\by) \right]
    &= \frac{\Cov_{\base}\left(r'(\bx,\by), e^{r(\bx, \by)/\lambda}\right)}{\EE_{\by\sim\base}\left[e^{r(\bx, \by)/\lambda}\right]} \\
    &= \Cov_{\base}\left(r'(\bx,\by), \frac{e^{r(\bx, \by)/\lambda}}{\EE_{\by\sim\base}\left[e^{r(\bx, \by)/\lambda}\right]}\right),
\end{split}
\end{equation}
thus concluding the proof.
\end{proof}

\estimatorConvergence*
\begin{proof}
The proof relies on the Weak Law of Large Numbers (WLLN) and the Continuous Mapping Theorem.

\textbf{1. Convergence of Component Estimators (WLLN):}

The WLLN states that for a sequence of i.i.d. random variables $X_1, \dots, X_n$ with a finite expectation $\mathbb{E}[X]$, the sample mean $\bar{X}_n = \frac{1}{n}\sum X_i$ converges in probability to the expectation: $\bar{X}_n \xrightarrow{p} \mathbb{E}[X]$.

We apply this to our three primary estimators (from Definition \ref{def:mc_estimator}), which are all sample means of i.i.d. variables (since each $y^{(i)}$ is drawn i.i.d. from $\pi_{\text{base}}$):

\begin{itemize}
    \item $ \hat{\mu}_{r'} \xrightarrow{p} \mathbb{E}_{\pi_{\text{base}}}[r'(y)] = \mu_{r'} $
    
    \item $ \hat{Z}_{r} \xrightarrow{p} \mathbb{E}_{\pi_{\text{base}}}[\exp(r(y) / \lambda)] = Z_{r} $
    
    \item $ \hat{C} \xrightarrow{p} \mathbb{E}_{\pi_{\text{base}}}[r'(y) \cdot \exp(r(y) / \lambda)] = C $
\end{itemize}
This holds because we assumed these expectations are finite.

\textbf{2. Convergence of Combined Estimators (Continuous Mapping):}

The Continuous Mapping Theorem states that for a sequence of random variables $X_n \xrightarrow{p} c$, and a function $g$ that is continuous at $c$, we have $g(X_n) \xrightarrow{p} g(c)$.

The covariance estimator $\widehat{\text{Cov}}$ is a continuous function $g_1$ of our three sample means:
$$ \widehat{\text{Cov}} = g_1(\hat{C}, \hat{\mu}_{r'}, \hat{Z}_{r}) = \hat{C} - (\hat{\mu}_{r'} \cdot \hat{Z}_{r}) $$
Since multiplication and subtraction are continuous functions, and all arguments converge in probability to their respective limits, the estimator $\widehat{\text{Cov}}$ converges in probability to the true covariance:
$$ \widehat{\text{Cov}} \xrightarrow{p} C - \mu_{r'} Z_{r} = \text{Cov}_{\pi_{\text{base}}}(r', \exp(r/\lambda)) $$

Finally, the gain estimator $\widehat{\Delta}_n(r, r')$ is a continuous function $g_2$ of $\widehat{\text{Cov}}$ and $\hat{Z}_{r}$:
$$ \widehat{\Delta}_n(r, r') = g_2(\widehat{\text{Cov}}, \hat{Z}_{r}) = \frac{\widehat{\text{Cov}}}{\hat{Z}_{r}} $$
Since division is a continuous function (given our assumption that the denominator $Z_{r} \neq 0$), $\widehat{\Delta}_n(r, r')$ converges in probability to the true gain $\Delta(r, r')$:
$$ \widehat{\Delta}_n(r, r') \xrightarrow{p} \frac{\text{Cov}_{\pi_{\text{base}}}(r', \exp(r/\lambda))}{Z_{r}} = \Delta(r, r') $$
Showing that $\widehat{\Delta}_n(r, r')$ converges to $\Delta(r, r')$.
\end{proof}

\klConvergence*
\begin{proof}
Taking logarithms in \eqref{eq:post_rlhf} gives $\log\frac{\rlhf(\by|\bx)}{\base(\by|\bx)} = r(\bx,\by)/\lambda - \log Z(\bx)$, so taking expectation under $\rlhf$ yields the population quantity
\begin{equation}
    \KL(\rlhf\|\base) = \frac{\EE_{\rlhf}[r]}{\lambda} - \log Z = \frac{\mu_r + \Delta(r,r)}{\lambda} - \log Z,
    \label{eq:kl_pop}
\end{equation}
where the second equality uses $\EE_{\rlhf}[r] = \EE_{\base}[r] + \Delta(r,r)$ from \cref{thm:characterization_of_reward_gain}.
The estimator $\widehat{\KL}(\rlhf\|\base)$ replaces each population quantity with its sample counterpart:
\begin{enumerate}
    \item $\hat{\mu}_r \xrightarrow{p} \mu_r$ by the WLLN.
    \item $\widehat{\Delta}_n(r,r) \xrightarrow{p} \Delta(r,r)$ by \cref{prop:estimator_consistency}, finiteness conditions are satisfied by hypothesis.
    \item $\widehat{Z}_r \xrightarrow{p} Z_r$ by the WLLN, since $\EE[\exp(r/\lambda)] = Z_r < \infty$ by hypothesis; then $\log\widehat{Z}_r \xrightarrow{p} \log Z_r$ by the Continuous Mapping Theorem, since $Z_r > 0$.
\end{enumerate}
The estimator is a continuous function of these three convergent quantities, so by a further application of the Continuous Mapping Theorem, $\widehat{\KL}(\rlhf\|\base) \xrightarrow{p} \KL(\rlhf\|\base)$.
\end{proof}

\subsection{Proof of Theoretical results in Section~\ref{sec:impact_of_noise}}
\label{apx:proof_sec4}

\ProxyRewardImpact*

\begin{proof}
Fix a query $\bx$ and work conditionally on $\bx$.
Define the two partition functions
\begin{equation}
    Z_\ast \triangleq \EE_{\completion\sim\base}\big[e^{\gr(\completion)/\lambda}\big],
    \qquad
    Z_\sim \triangleq \EE_{\completion\sim\base}\big[e^{\nr(\completion)/\lambda}\big],
\end{equation}
and recall that $\calN(\completion) = \nr(\completion) - \gr(\completion)$ is the preference residual.
Applying \cref{thm:characterization_of_reward_gain} with alignment reward $r = \gr$ and evaluation reward $r' = \gr$, and again with $r = \nr$ and $r' = \gr$, gives
\begin{align}
    \EE_{\gmodel}[\gr] - \EE_{\base}[\gr] &= \Cov_{\completion\sim\base}\!\left(\gr(\completion),\,\tfrac{e^{\gr(\completion)/\lambda}}{Z_\ast}\right),\\
    \EE_{\nmodel}[\gr] - \EE_{\base}[\gr] &= \Cov_{\completion\sim\base}\!\left(\gr(\completion),\,\tfrac{e^{\nr(\completion)/\lambda}}{Z_\sim}\right).
\end{align}
Subtracting the second from the first, the $\EE_{\base}[\gr]$ terms cancel, and linearity of the covariance in the second argument yields
\begin{equation}
    \EE_{\gmodel}[\gr] - \EE_{\nmodel}[\gr]
    = \Cov_{\completion\sim\base}\!\left(\gr(\completion),\,\frac{e^{\gr(\completion)/\lambda}}{Z_\ast} - \frac{e^{\nr(\completion)/\lambda}}{Z_\sim}\right).
    \label{eq:prop3_subtract}
\end{equation}
Algebraic factorization of the second argument and the substitution $\calN = \nr - \gr$ give
\begin{equation}
    \frac{e^{\gr/\lambda}}{Z_\ast} - \frac{e^{\nr/\lambda}}{Z_\sim}
    = \frac{e^{\gr/\lambda}}{Z_\ast}\left(1 - \frac{Z_\ast}{Z_\sim}\,e^{(\nr - \gr)/\lambda}\right)
    = \frac{e^{\gr/\lambda}}{Z_\ast}\left(1 - \frac{Z_\ast}{Z_\sim}\,e^{\calN(\completion)/\lambda}\right).
\end{equation}
Substituting back into \eqref{eq:prop3_subtract} yields
\begin{equation}
    \EE_{\gmodel}[\gr] - \EE_{\nmodel}[\gr]
    = \Cov_{\completion\sim\base}\!\left(\gr(\completion),\,\frac{e^{\gr(\completion)/\lambda}}{Z_\ast}\left(1 - \frac{Z_\ast}{Z_\sim}\,e^{\calN(\completion)/\lambda}\right)\right).
    \label{eq:before_change_of_vars}
\end{equation}

Next, we express $Z_\sim/Z_\ast$ under $\gmodel$:
\begin{align}
    \frac{Z_\sim}{Z_\ast}
    &= \frac{
        \EE_{\base}\big[e^{(\gr+\calN)/\lambda}\big]
    }{
        \EE_{\base}\big[e^{\gr/\lambda}\big]
    } \\
    &= \frac{1}{Z_\ast}
       \sum_{\completion\in\calV}
       \base(\completion)\,
       e^{\gr(\completion)/\lambda}\,e^{\calN(\completion)/\lambda} \\
    &= \EE_{\completion\sim\gmodel}
       \big[e^{\calN(\completion)/\lambda}\big].
\end{align}
Hence
\begin{equation}
    \frac{Z_\ast}{Z_\sim}
    = \frac{1}{
        \EE_{\completion\sim\gmodel}
        \big[e^{\calN(\completion)/\lambda}\big]
      }.
      \label{eq:change_of_var}
\end{equation}

Substituting \cref{eq:change_of_var} into \cref{eq:before_change_of_vars},
\begin{equation}
    \EE_{\gmodel}[\gr] - \EE_{\nmodel}[\gr]
    = \Cov_{\completion\sim\base}\!\left(\gr(\completion),\,\frac{e^{\gr(\completion)/\lambda}}{Z_\ast}\left(1 - \frac{e^{\calN(\completion)/\lambda}}{ \EE_{\completion\sim\gmodel}
        \big[e^{\calN(\completion)/\lambda}\big]}\right)\right),
\end{equation}
which completes our proof.
\end{proof}

\lipsBound*
\begin{proof}
Fix a query $\bx$ and work conditionally on $\bx$.
By the tilted-form characterization,
\begin{equation}
    \pi(r,\lambda)(\by \mid \bx)
    = \frac{\base(\by \mid \bx)\, e^{r(\bx,\by)/\lambda}}
           {\EE_{\base}[e^{r(\bx,\by)/\lambda}]}.
\end{equation}
Define
\begin{equation}
    Z_r \triangleq \EE_{\base}[e^{r/\lambda}],
    \qquad
    w_r(\bx,\by) \triangleq \frac{e^{r(\bx,\by)/\lambda}}{Z_r},
\end{equation}
(and similarly $Z_s,w_s$ for $s$; dependence on $\bx$ is implicit).

The fundamental limit (covariance) identity applied with alignment
reward $r$ and evaluation reward $r^*$ gives
\begin{equation}
    \Delta(r,r^*)
    = \EE_{\pi(r,\lambda)}[r^*] - \EE_{\base}[r^*]
    = \Cov_{\base}(r^*, w_r).
\end{equation}
Similarly,
\begin{equation}
    \Delta(s,r^*) = \Cov_{\base}(r^*, w_s).
\end{equation}
Therefore
\begin{equation}
    \Delta(r,r^*) - \Delta(s,r^*)
    = \Cov_{\base}\big(r^*, w_r - w_s\big).
    \label{eq:delta_diff_cov_lambda}
\end{equation}

By Cauchy-Schwarz,
\begin{align}
    |\Delta(r,r^*) - \Delta(s,r^*)| &= |\Cov_{\base}\big(r^*, w_r - w_s\big)|  \\
    & \leq \Var(r^*)^{1/2} \Var(w_r - w_s)^{1/2} \\
    & \leq \Var(r^*)^{1/2} \EE[(w_r - w_s)^2]^{1/2} \\
    & = \Var(r^*)^{1/2} \|w_r - w_s\|_{L^2(\base)}
\end{align}
We now bound $\|w_r - w_s\|_{L^2(\base)}$ in terms of $\|r-s\|_{L^2(\base)}$.
Again by Cauchy-Schwarz,
\begin{align}
    |\Delta(r,r^*) - \Delta(s,r^*)| &= |\Cov_{\base}\big(r^*, w_r - w_s\big)|  \\
    & \leq \Var(r^*)^{1/2} \Var(w_r - w_s)^{1/2}
    \label{eq:delta_bound_wr_ws_lambda}
\end{align}
We now want to bound the term $V_{r,s} =\Var[w_r - w_s]^{1/2}$ with $\EE[(r-s)^2]$. Notice that $\EE[w_r]=\EE[w_s]=1$, as they are both normalized quantities, implying that $\Var[w_r - w_s]=\EE[(w_r-w_s)^2]$. We define $w_t$ as the interpolation
\begin{equation}
    w_t = \frac{e^{\frac{s+(r-s)t}{\lambda}}}{Z_t},\qquad\text{with}\qquad Z_t=\EE[e^{\frac{s+(r-s)t}{\lambda}]}.
\end{equation}
By the fundamental theorem of calculus, we have
\begin{equation}
    w_r-w_s = \int_{0}^{1} \left(\frac{d}{dt}w_t\right)\,dt
    = \frac{1}{\lambda}\int_{0}^{1}w_t((r-s)-\EE[w_t(r-s)])\,dt,
\end{equation}
which we can directly substitute inside $\EE[(w_r-w_s)^2]$:
\begin{equation}
\begin{split}
     V_{r,s} &= \EE[(w_r-w_s)^2]
     = \frac{1}{\lambda^2}\EE\left[\left(\int_{0}^{1}w_t((r-s)-\EE[w_t(r-s)])\,dt\right)^2\right]\\
     &\leq \frac{1}{\lambda^2} \EE\left[\int_0^1 w_t^2 \left((r-s)-\EE[w_t(r-s)]\right)^2\,dt\right]\\
     &= \frac{1}{\lambda^2} \int_0^1 \EE\left[w_t^2 \left((r-s)-\EE[w_t(r-s)]\right)^2\right]\,dt\\
     &\leq \frac{W}{\lambda^2} \int_0^1 \EE\left[w_t \left((r-s)-\EE[w_t(r-s)]\right)^2\right]\,dt\\
     &= \frac{W}{\lambda^2} \int_0^1 \Var_{w_t}\left[r-s\right]\,dt,
\end{split}
\end{equation}
where we applied Cauchy-Schwarz to swap square and integral, Fubini-Tonelli to swap $\EE$ and $\int_0^1$, and a bound on the tilting function, $W=\sup_{t\in[0,1]} \|w_t\|_{\infty}$, to extract a $w_t$ factor from within $\EE$. The constant $W$ admits two complementary bounds depending on the assumption we can make on our reward models. If the rewards are bounded, i.e., $|r(\bx,\by)|, |s(\bx,\by)| \leq B$ for all $\by$, then also $|r_t| = |(1-t)s + tr| \leq B$, giving
\begin{equation}
    w_t(\by) = \frac{e^{r_t(\by)/\lambda}}{Z_t}
    \leq \frac{e^{B/\lambda}}{e^{-B/\lambda}} = e^{2B/\lambda} \eqqcolon W_B,
\end{equation}
since the numerator is at most $e^{B/\lambda}$ and the partition function satisfies $Z_t = \EE[e^{r_t/\lambda}] \geq e^{-B/\lambda}$.
When rewards are unbounded, we can instead exploit the fact that $Z_t$ contains the term $\base(\by)\,e^{r_t(\by)/\lambda}$, so that
\begin{equation}
    w_t(\by) = \frac{e^{r_t(\by)/\lambda}}{Z_t}
    \leq \frac{e^{r_t(\by)/\lambda}}{\base(\by)\,e^{r_t(\by)/\lambda}}
    = (\base(\by))^{-1}  \eqqcolon W_U,
\end{equation}
which depends only on the base model and is always finite when $\base$ has finite full support -- as is the case for language models with finite vocabulary and bounded generation length. In practice, then, we can write
\begin{equation}
    W=\sup_{t\in[0,1]} \|w_t\|_{\infty}\leq\min\left\{W_B, W_U\right\}=\min\left\{e^{2B/\lambda}, (\base(\by))^{-1}\right\}.
\end{equation}

Notice moreover that a property of the exponential family of distributions is that its $\log$-partition is the cumulant generating function, and in particular
\begin{equation}
    \frac{d^2}{dt^2}\log Z_t = \frac{1}{\lambda}\frac{d}{dt}\EE[w_t(r-s)] = \frac{1}{\lambda^2}\Var_{w_t}[r-s],
\end{equation}
which allows us to re-write
\begin{equation}
    V_{r,s} \leq \frac{W}{\lambda} \int_0^1 \frac{1}{\lambda}\Var_{w_t}\left[r-s\right]\,dt
    = \frac{W}{\lambda} \int_0^1 \frac{d}{dt}\EE[w_t(r-s)]\,dt
    = \frac{W}{\lambda} \EE[(w_r-w_s)(r-s)].
\end{equation}
Applying Cauchy-Schwarz once again gives
\begin{equation}
    V_{r,s} = \Var[w_r-w_s]\leq \frac{W}{\lambda} \EE[(w_r-w_s)(r-s)] \leq \frac{W}{\lambda} \sqrt{\Var[w_r-w_s]\Var[r-s]}.
\end{equation}
The variance of $r-s$ can be bound by its second moment. Dividing both sides by $\sqrt{\Var[w_r-w_s]}$ provides the bound we are looking for:
\begin{equation}
    \sqrt{\Var[w_r-w_s]} \leq \frac{W}{\lambda} \sqrt{\EE[(r-s)^2]} = \frac{W}{\lambda} \|r-s\|_{L^2(\base)}.
\end{equation}
Combining with \cref{eq:delta_bound_wr_ws_lambda} completes the proof:
\begin{equation}
    |\Delta(r,r^*) - \Delta(s,r^*)| \leq \Var(r^*)^{1/2} \Var(w_r - w_s)^{1/2} \leq \frac{W}{\lambda} \Var(r^*)^{1/2} \|r-s\|_{L^2(\base)}.
\end{equation}

\end{proof}

\ensembleRate*

\begin{proof}
Let $\Delta(r)$ denote the post-RLHF true reward when we align with
reward $r$, i.e.,
\[
\Delta(r) \;\triangleq\; \EE_{\pi(r,\lambda)}[\gr].
\]
In particular,
\[
\Delta(R_n) = \EE_{\pi(R_n,\lambda)}[\gr],
\qquad
\Delta(\gr) = \EE_{\pi(\gr,\lambda)}[\gr]
            = \EE_{\gmodel}[\gr].
\]

By the deterministic Lipschitz bound (proved earlier), there exists a
constant $C>0$ such that
\begin{equation}
    \big|\Delta(r) - \Delta(\gr)\big|
    \;\le\; C \times \Var(r^*)^{1/2}\|r - \gr\|_{L^2(\base)}
    \label{eq:lipschitz_delta}
\end{equation}
for all reward functions $r$, where
$\| \cdot \|_{L^2(\base)}$ is the $L^2$ norm under the base model $\base$.

Apply \eqref{eq:lipschitz_delta} with $r=R_n$:
\begin{equation}
    \big|\Delta(R_n) - \Delta(\gr)\big|
    \;\le\; C \times \Var(r^*)^{1/2} \|R_n - \gr\|_{L^2(\base)}.
    \label{eq:lipschitz_Rn}
\end{equation}

We now bound the right-hand side in expectation over the randomness of the
ensemble $\{r_i\}_{i=1}^n$.
Work in the Hilbert space $L^2(\base)$ with inner product
$\langle f,g\rangle = \EE_{\base}[f g]$.
Define the centered random elements
\[
\Delta_i \;\triangleq\; r_i - \gr,
\qquad
R_n - \gr
= \frac{1}{n}\sum_{i=1}^n \Delta_i.
\]

First compute the second moment of the $L^2(\base)$-norm:
\begin{align}
\EE\big\|R_n - \gr\big\|_{L^2(\base)}^2
&= \EE\Big\langle \frac{1}{n}\sum_{i=1}^n \Delta_i,
                      \frac{1}{n}\sum_{j=1}^n \Delta_j
           \Big\rangle \nonumber\\
&= \frac{1}{n^2}
   \sum_{i,j=1}^n \EE\langle \Delta_i,\Delta_j\rangle.
\label{eq:second_moment_expand}
\end{align}

Since $r_1,\dots,r_n$ are i.i.d.\ with mean $\gr$ in $L^2(\base)$,
we have $\EE[\Delta_i]=0$ and, for $i\neq j$,
\[
\EE\langle \Delta_i,\Delta_j\rangle
= \langle \EE\Delta_i,\EE\Delta_j\rangle
= 0.
\]
Thus only the diagonal terms survive in
\eqref{eq:second_moment_expand}:
\begin{align}
\EE\big\|R_n - \gr\big\|_{L^2(\base)}^2
&= \frac{1}{n^2}
   \sum_{i=1}^n \EE\|\Delta_i\|_{L^2(\base)}^2 \nonumber\\
&= \frac{1}{n^2}\cdot n\,
   \EE\|\Delta_1\|_{L^2(\base)}^2 \nonumber\\
&= \frac{1}{n}\,
   \EE\|r_1 - \gr\|_{L^2(\base)}^2.
\end{align}
Set
\[
\sigma^2 \;\triangleq\;
\EE\|r_1 - \gr\|_{L^2(\base)}^2 < \infty,
\]
which is finite because each $r_i$ is bounded.
Then
\begin{equation}
\EE\big\|R_n - \gr\big\|_{L^2(\base)}^2
= \frac{\sigma^2}{n}.
\label{eq:variance_Rn}
\end{equation}

By Cauchy--Schwarz,
\begin{equation}
\EE\big\|R_n - \gr\big\|_{L^2(\base)}
\;\le\;
\Big(\EE\big\|R_n - \gr\big\|_{L^2(\base)}^2\Big)^{1/2}
\;\le\; \frac{\sigma}{\sqrt{n}}.
\label{eq:first_moment_bound}
\end{equation}

Taking expectations in \eqref{eq:lipschitz_Rn} and using
\eqref{eq:first_moment_bound} gives
\begin{equation}
\EE\big|\Delta(R_n) - \Delta(\gr)\big|
\;\le\; C \times \Var(r^*)^{1/2} \EE\big\|R_n - \gr\big\|_{L^2(\base)}
\;\le\; \frac{C\Var(r^*)^{1/2}  \sigma}{\sqrt{n}}.
\end{equation}
In other words,
\[
\EE\big|\EE_{\pi(R_n,\lambda)}[\gr]
      - \EE_{\gmodel}[\gr]\big|
= O\!\left(\frac{\Var(r^*)^{1/2} }{\sqrt{n}}\right),
\]
which yields the claimed $O(n^{-1/2})$ rate.
\end{proof}

\section{On the need for a new paradigm for reward modeling}
\label{apx:leaveBTbehind}
Proxy reward models are predominantly trained using the Bradley-Terry model, which frames human feedback strictly as binary pairwise comparisons over relative rankings, and thus fails to impose restrictions on preference residuals. Specifically, under the BT framework, the probability that a response $\by_w$ is preferred over a response $\by_l$ given a prompt $\bx$ is modeled as $\Pr(\by_w \succ \by_l \mid \bx) = \sigma(\nr(\bx, \by_w) - \nr(\bx, \by_l))$, where $\sigma$ is the logistic function.
However, because this objective captures only the relative ordering of preferences, it provides no structural guarantee that the learned proxy reward $\nr$ is quantitatively close to the unobserved gold reward $\gr$ in terms of preference residual. 
Hence, a reward can perfectly minimize the BT loss with an arbitrarily large preference residual $\calN(\bx,\by)$.

The result in \cref{prop:reward_gap} shows that preference residual ---uncontrolled in the BT framework--- causes an exponential decrease in the gold reward, leading to reward hacking.
Conversely, \cref{prop:lipschitz_bound} provides justification that if the preference residual is small, then reward hacking is mitigated, i.e., gains in performance are transferred to the evaluation using the gold reward.

Together, these results highlight \emph{the necessity for a new reward modeling paradigm}: we \textbf{must} incorporate preference magnitude to directly anchor the proxy reward, allowing us to actively measure and minimize its absolute $\ell_2$ distance to the gold reward.
Incorporating the magnitude of a preference would anchor the scale of the proxy, providing a direct mechanism to measure and minimize the preference residual.

\paragraph{Example: BT-invariant but increases the proxy--gold gap.}
For a prompt $\bx$ suppose that the set of candidate responses decomposes into two disjoint clusters
\begin{equation}
\calV = \calV_A(\bx)\cup \calV_B(\bx), \qquad \calV_A(\bx)\cap \calV_B(\bx)=\varnothing,
\end{equation}
where $\calV$ is the set of all possible responses. Also, consider that BT preference dataset for $\bx$ contains only within-cluster comparisons, i.e., every labeled pair $(\by^w,\by^l)$ satisfies either $\by^w,\by^l\in\calV_A(\bx)$ or $\by^w,\by^l\in\calV_B(\bx)$.

\paragraph{When do clustered preference clusters arise?} 
This two-component structure naturally emerges in practice from mixture-of-sources candidate generation paired with within-source preference labeling. The following are examples of when this emerges:
\begin{enumerate}
    \item \textbf{Helpfulness and Harmlessness} We can generate one batch of candidates using a ``helpful'' decoding regime (cluster $\calV_A(\bx)$) and a second batch using a ``safe/refusal'' regime (cluster $\calV_B(\bx)$), as in HH-RLHF \cite{bai2022traininghelpfulharmlessassistant}, where it is useful to train a \emph{helpful} and \emph{harmless} reward. 
    \item \textbf{Subtle Safety Risks} To reduce annotation costs, alignment pipelines may use a \emph{selective mechanism} where an AI labels simpler candidate pairs while humans evaluate only the most challenging ones. This inadvertently creates isolated data clusters.
    \item \textbf{Reasoning} A model might generate one batch of candidates using Chain-of-Thought reasoning, and another batch using direct answers. If human raters compare only CoT responses to other CoT responses (to judge the logic) and direct answers to direct answers (to judge conciseness), we have two clusters.
    \item \textbf{Tool Use} When training an LLM to use external tools (like web search), we may evaluate tool-use trajectories against other tool-uses to improve tool usage. These are never compared against a model simply answering from its own internal weights.
\end{enumerate}

In these scenarios, the BT model learns accurate within-cluster rankings but leaves the relative reward offset between the clusters weakly identified. 
The shift preserves all observed BT pairwise differences while potentially inflating the absolute preference residual $\calN$. 

\paragraph{The theoretical Argument.}
For any $M>0$, define a modified proxy reward
$\tilde r^{(M)}(\bx,\by)
\triangleq
\gr(\bx,\by)
+ M\,\mathbf{1}\{\by\in\calV_A(\bx)\}
- M\,\mathbf{1}\{\by\in\calV_B(\bx)\}.
$
For any observed preference pair $(\by^w,\by^l)$, the pair lies within a single cluster; hence the $\pm M$ offsets cancel. Therefore every BT likelihood term is unchanged, and consequently $\calL_{\mathrm{BT}}(\tilde r^{(M)})=\calL_{\mathrm{BT}}(\gr)$, which is the ``optimal''.

However, note that the preference residual can be written as
\begin{equation}
\calN(\bx,\by)\triangleq \nr^{(M)}(\bx,\by)-\gr(\bx,\by) = M\,\mathbf{1}\{\by\in\calV_A(\bx)\} - M\,\mathbf{1}\{\by\in\calV_B(\bx)\}.
\end{equation}
Hence $\|\calN^{(M)}\|_{L^2(\base)}$ (and $\sup_{\by}|\calN^{(M)}(\bx,\by)|$) can be made arbitrarily large by increasing $M$, while the BT training objective remains identical.
At the same time, Proposition~\ref{prop:reward_gap} shows that the proxy-gold degradation is governed by the exponential tilt given as
\begin{equation}
\frac{e^{\calN(\bx,\by)/\lambda}}{\EE_{\pi^*}\big[e^{\calN(\bx, \completion)/\lambda}\big]}
=
\frac{e^{M/\lambda}\,\mathbf{1}\{\by\in\calV_A(\bx)\} + e^{-M/\lambda}\,\mathbf{1}\{\by\in\calV_B(\bx)\}}{\pi^*(\calV_A(\bx))e^{M/\lambda} +   \pi^*(\calV_B(\bx))e^{-M/\lambda}} > 0,
\end{equation}
which amplifies the mismatch between the proxy-aligned model $\pi(\nr^{(M)},\lambda)$ and the gold-aligned model $\gmodel$, thereby increasing the gap $\delta$ in Proposition~\ref{prop:reward_gap}.

\section{Reward-KL Pareto Frontier}
\label{apx:ParetoFront}

\cref{fig:pareto_convergence} validates the practicality of our proposed estimators for the Pareto frontier by showing the convergence of the entire theoretical Reward-KL Pareto frontier as a function of the number of base model samples $n$. 
Figure \ref{fig:pareto_convergence} illustrates the estimated frontiers using $n \in \{10, 100, 1000, 10000\}$ across both tasks (safety and summarization). 
We observe that while small sample sizes (e.g., $n=10$ or $100$) systematically underestimate the achievable reward and prematurely truncate in the high-KL regime, the frontier rapidly stabilizes as $n$ increases. 
This behavior directly reflects the mechanics of the estimators in \cref{def:mc_estimator} and \cref{def:kl_estimator}: in the high-KL regime (which corresponds to a small penalty $\lambda$), the exponential tilt $e^{r/\lambda}$ amplifies the variance of high-reward samples, thereby requiring a larger number of samples to achieve a tight estimate. 
Crucially, the curves for $n=1000$ and $n=10000$ are nearly overlapping across most of the KL spectrum. 
This macro view of the frontier's convergence complements the pointwise convergence guarantees of \cref{prop:estimator_consistency} and \cref{prop:kl_consistency}, as well as the fixed-$\lambda$ empirical results in \cref{fig:estimation_gain} and \cref{fig:estimation_kl}. 
Ultimately, \cref{fig:pareto_convergence} confirms that a few thousand samples (around one thousand) are sufficient to reliably delineate the fundamental limits of alignment from base model samples alone.

\begin{figure}[htb!]
    \centering
    \begin{subfigure}[b]{0.45\textwidth}
        \centering
        \includegraphics[width=\linewidth]{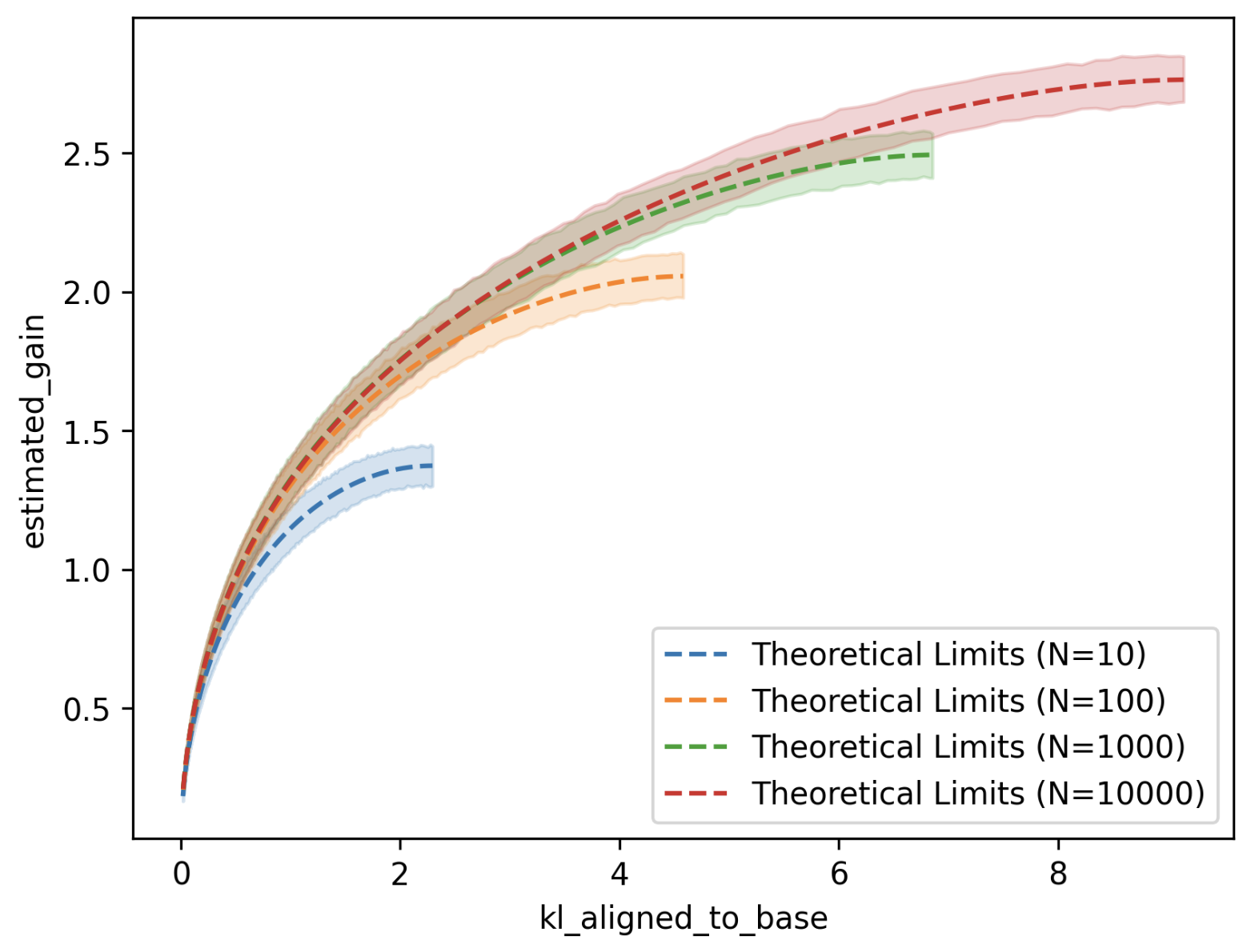}
        \caption{Zephyr 7b with BeaverTails}
    \end{subfigure}
    \hfill
    \begin{subfigure}[b]{0.45\textwidth}
        \centering
        \includegraphics[width=\linewidth]{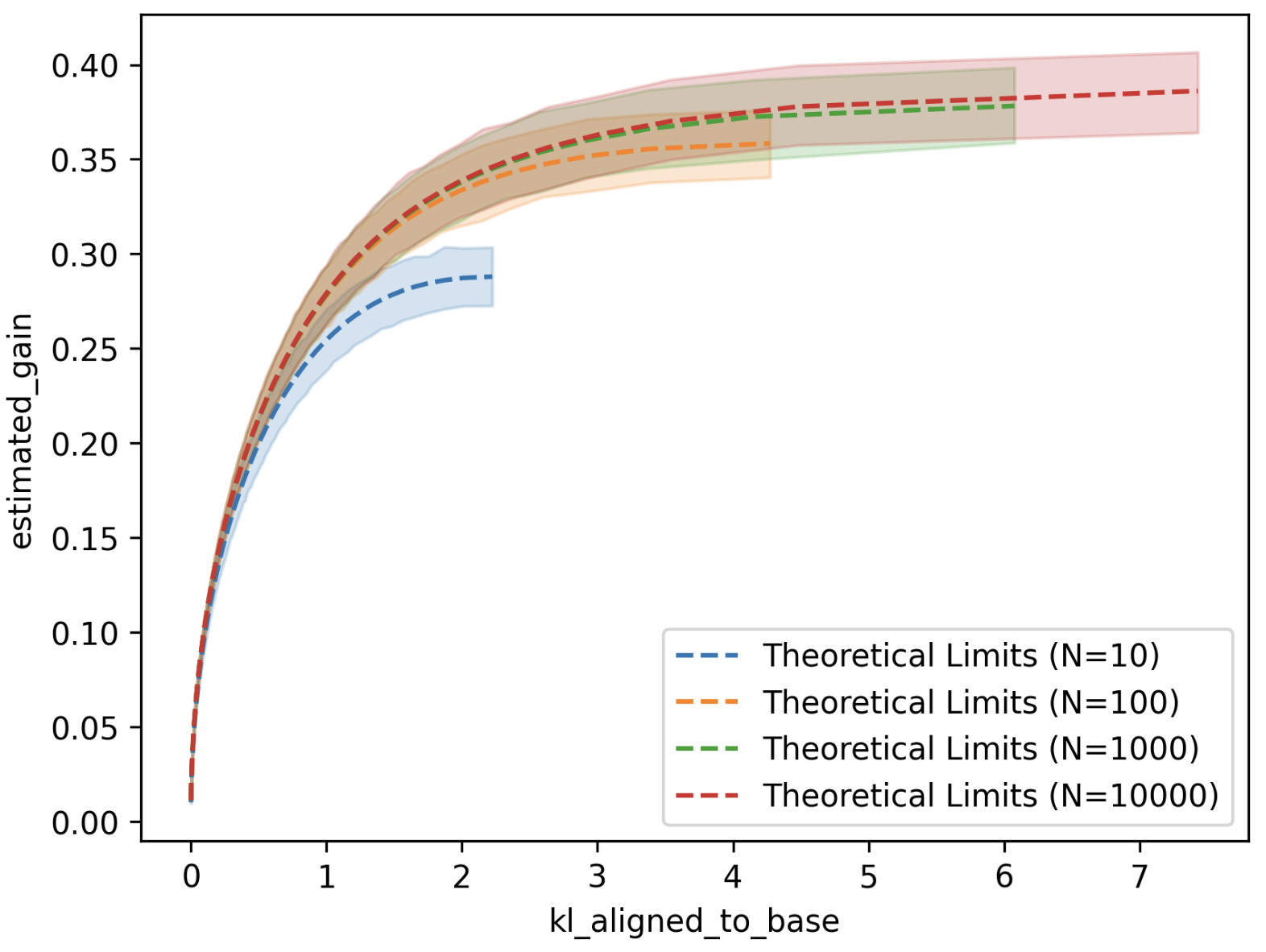}
        \caption{Pythia 1B with TL;DR}
    \end{subfigure}
    \caption{
    \textbf{Convergence of Pareto front.} Estimated Reward gain (\cref{def:mc_estimator}) vs. KL divergence (\cref{def:kl_estimator}) Pareto frontier using different number of samples per prompt. Shaded bands show 95\% confidence intervals using bootstrap from Seaborn \cite{Waskom2021}.
    }
    \label{fig:pareto_convergence}
\end{figure}

\section{Experimental Runs Hyper-parameter Sweeps}
\label{apx:hparam_sweeps}

To achieve the experimental results presented in the paper for GRPO and PPO, we ran an extensive hyper-parameter sweep. The reward gain per KL value for each run in the sweep are displayed in Figures~\ref{fig:beavertails_grpo_sweep}, ~\ref{fig:beavertails_ppo_sweep}, ~\ref{fig:tldr_grpo_sweep}, and ~\ref{fig:tldr_ppo_sweep} for training with Beavertails using GRPO, Beavertails using PPO, TL;DR using GRPO, and TL;DR using PPO, respectively. The GRPO and PPO results with the highest reward gains with smooth learning curves are displayed in the main body in Figure.~\ref{fig:fundamental_limit}. The parameters and maximum reward gain for the KL values across hyper-parameter sweeps can be found in Table~\ref{tab:combined_sweep} and Table~\ref{tab:combined_sweep_tldr} for tuning Zephyr with Beavertails and Pythia with TL;DR, respectively. The run from each fine-tuning method and model/dataset combination that are shown in the main body are in \textbf{bold}.

\begin{table}[htbp]
    \centering

    \begin{subtable}[t]{0.48\textwidth}
        \centering
        \begin{tabular}{l c c c | c}
            \toprule
            \textbf{$\beta$} & \textbf{LR} & \textbf{BS} & \textbf{Num Gens} & \textbf{$\max R \Delta$ } \\
            \midrule
            2e-2 & 5e-7 & 256 & 16 &2.101 \\
            \textbf{2e-2} &\textbf{ 3e-7 }& \textbf{256} & \textbf{16} & \textbf{2.117}\\
            1.5e-2 & 3e-7 & 256 & 16 & 2.103 \\
            1e-2 & 5e-7 & 256 & 64 & 2.083  \\
            1e-2 & 5e-7 & 256 & 16 & 2.066 \\
            1e-2 & 3e-7 & 256 & 16 & 2.088 \\
            1e-2 & 2e-7 & 256 & 16 & 2.104 \\
            5e-3 & 5e-7 & 256 & 16 & 2.034 \\
            3e-3 & 5e-7 & 256 & 16 & 2.051 \\
            1e-2 & 2e-7 &  256 & 16 & 1.909 \\
            1.5e-2 &  2e-7 &  256 & 16& 1.886\\
            1e-2 & 2e-7  &  256 & 8& 1.899 \\
            2e-2 & 2e-7  &  256 &16 & 1.889 \\
            \bottomrule
        \end{tabular}
        \caption{GRPO}
        \label{tab:grpo_sweep}
    \end{subtable}\hfill
    \begin{subtable}[t]{0.48\textwidth}
        \centering
        \begin{tabular}{ccccc|c}
        \toprule
        $\beta$ & \textbf{LR} & \textbf{BS} & $T$ & \textbf{PPO Epochs} & $\max R\Delta$ \\
        \midrule
        5e-2 & 3e-6 & 32 & 1.0 & 1 & 0.918 \\
        1e-3 & 1e-6 & 32 & 1.0 & 1 & 1.573 \\
        5e-2 & 3e-6 & 32 & 0.7 & 2 & 1.809 \\
        5e-2 & 1e-6 & 32 & 1.0 & 1 & 1.908 \\
        5e-2 & 1e-6 & 32 & 1.0 & 2 & 1.924 \\
        3e-2 & 1e-6 & 32 & 1.0 & 1 & 1.232 \\
        \textbf{2e-2} &\textbf{ 1e-6} & \textbf{32} & \textbf{0.7 }& \textbf{1} & \textbf{2.084} \\
        1e-2 & 1e-6 & 32 & 0.7 & 2 & 1.544 \\
        2e-2 & 1e-6 & 32 & 0.7 & 2 & 1.698 \\
        1e-2 & 3e-6 & 32 & 0.7 & 2 & 1.810 \\
        1e-2 & 1e-6 & 32 & 0.5 & 1 & 1.357 \\
        2e-2 & 3e-6 & 32 & 0.7 & 2 & 1.411 \\
        2e-2 & 3e-6 & 32 & 0.5 & 2 & 1.809 \\
        \bottomrule
        \end{tabular}
        \caption{PPO}
        \label{tab:beavertails_sweep}
    \end{subtable}
    
    \vspace{0.2cm} %
    \caption{The hyper parameters swept for training Zephyr-7B-SFT-full on the BeaverTails dataset with \textbf{GRPO} and \textbf{PPO}. The maximum reward gain is reported within the KL range that the theoretical limits are calculated for [0, 9.14]. The hyper-parameter run that is used in Figure~\ref{fig:fundamental_limit} is \textbf{bolded}. For comparison, the maximum reward gain for the theoretical limit and BoN are 2.764 and 2.722, respectively.}
    \label{tab:combined_sweep}
\end{table}

\begin{table}[]
    \centering

    \begin{subtable}[t]{0.48\textwidth}
        \centering
                \begin{tabular}{cccc|c}
        \toprule
        $\beta$ & \textbf{LR} & \textbf{BS} & \textbf{Num Gens} & max $R\Delta$ \\
        \midrule
        1e-5 & 1e-6 & 128 & 16 & 0.175 \\
        1e-1 & 1e-6 & 128 & 4  & 0.251 \\
        1e-1 & 1e-6 & 128 & 4  & 0.253 \\
        1e-1 & 1e-6 & 128 & 4  & 0.256 \\
        1e-1 & 1e-6 & 128 & 4  & 0.249 \\
        1e-2 & 1e-6 & 128 & 8  & 0.243 \\
        1e-3 & 1e-6 & 128 & 8  & 0.220 \\
        1e-3 & 1e-6 & 128 & 4  & 0.196 \\
        1e-4 & 1e-6 & 128 & 8  & 0.225 \\
        1e-1 & 1e-6 & 128 & 4  & 0.277 \\
        5e-2 & 1e-6 & 128 & 8  & 0.276 \\
        \textbf{7e-2} & \textbf{1e-6 }& \textbf{128} & \textbf{8}  & \textbf{0.282} \\
        5e-2 & 5e-7 & 128 & 8  & 0.272 \\
        3e-2 & 1e-6 & 128 & 8  & 0.257 \\
        5e-2 & 1e-6 & 128 & 16 & 0.273 \\
        \bottomrule
        \end{tabular}
        \caption{GRPO}
        \label{tab:grpo_sweep_tldr}
    \end{subtable}\hfill
    \begin{subtable}[t]{0.48\textwidth}
        \centering
        \begin{tabular}{ccccc|c}
        \toprule
        $\beta$ & \textbf{LR} & \textbf{BS} & $T$ & \textbf{PPO Epochs} & max $R\Delta$ \\
        \midrule
        2e-2   & 3e-6 & 64  & 0.7 & 4 & 0.114 \\
        1.5e-2 & 3e-6 & 64  & 0.7 & 4 & 0.147 \\
        1e-2   & 3e-6 & 32  & 0.7 & 4 & 0.113 \\
        \textbf{1e-2}   & \textbf{3e-6} & \textbf{64}  & \textbf{0.7 }& \textbf{2} & \textbf{0.187} \\
        1e-2   & 3e-6 & 64  & 1.0 & 4 & 0.145 \\
        1e-2  & 3e-6 & 64  & 0.7 & 6 & 0.214 \\
        1.5e-2 & 3e-6 & 64  & 0.7 & 4 & 0.157 \\
        7e-3   & 3e-6 & 64  & 0.7 & 4 & 0.163 \\
        5e-3   & 3e-6 & 64  & 0.7 & 4 & 0.011 \\
        5e-3   & 5e-6 & 64  & 0.7 & 4 & 0.012 \\
        3e-3   & 3e-6 & 64  & 0.7 & 4 & 0.022 \\
        7e-3   & 3e-6 & 128 & 0.7 & 4 & 0.167 \\
        7e-3   & 3e-6 & 64  & 0.3 & 4 & 0.131 \\
        7e-3   & 5e-7 & 64  & 0.7 & 4 & 0.212 \\
        1e-2   & 1e-5 & 64  & 0.7 & 4 & 0.143 \\
        \bottomrule
        \end{tabular}
        \caption{PPO}
        \label{tab:tldr_sweep}
    \end{subtable}
    
    \vspace{0.2cm} %
    \caption{The hyper parameters swept for training Pythia 1B SFT on the TL;DR dataset with \textbf{GRPO} and \textbf{PPO}. The maximum reward gain is reported within the KL range that the theoretical limits are calculated for [0, 7.96]. The hyper-parameter run that is used in Figure~\ref{fig:fundamental_limit} is \textbf{bolded}. For comparison, the maximum reward gain for the theoretical limit and BoN are 0.387 and 0.385, respectively. For PPO, the two highest estimated reward gain runs were not used in the main body due to how noisy they are. (See Fig.~\ref{fig:tldr_ppo_sweep} for more details)}
    \label{tab:combined_sweep_tldr}
\end{table}

\begin{figure}[thb!]
    \centering
    \includegraphics[width=\linewidth]{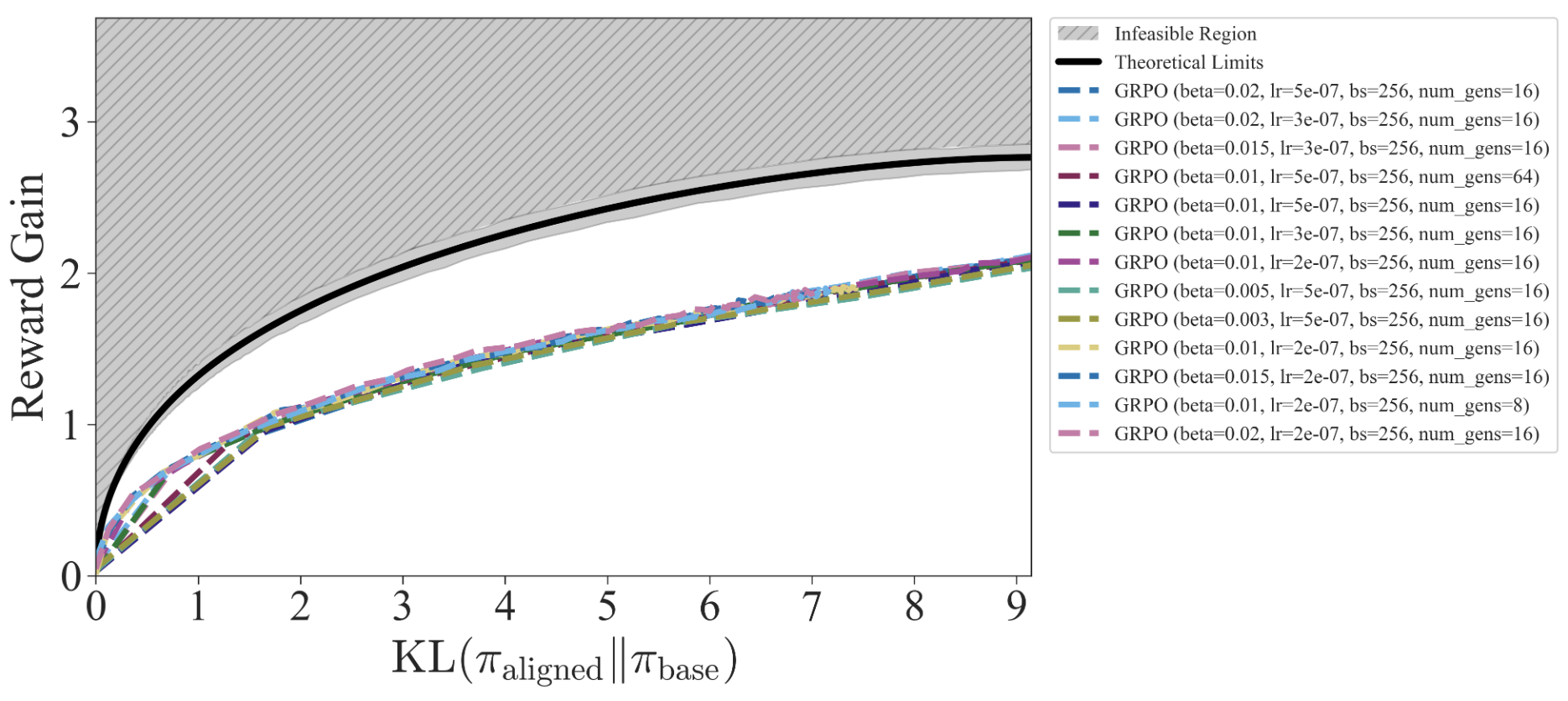}
\caption{\textbf{Hyperparameter sweep for the fine-tuning Zephyr 7B SFT Full using GRPO on the Beavertails dataset. }The plot shows the reward gain as the trained deviates from the base model. Abnormalities in the plotted lines are caused by training runs either plateauing before a KL of 9.14 or from oscillating between a small range of KL values.}
    \label{fig:beavertails_grpo_sweep}
\end{figure}

\begin{figure}[htb!]
    \centering
    \includegraphics[width=\linewidth]{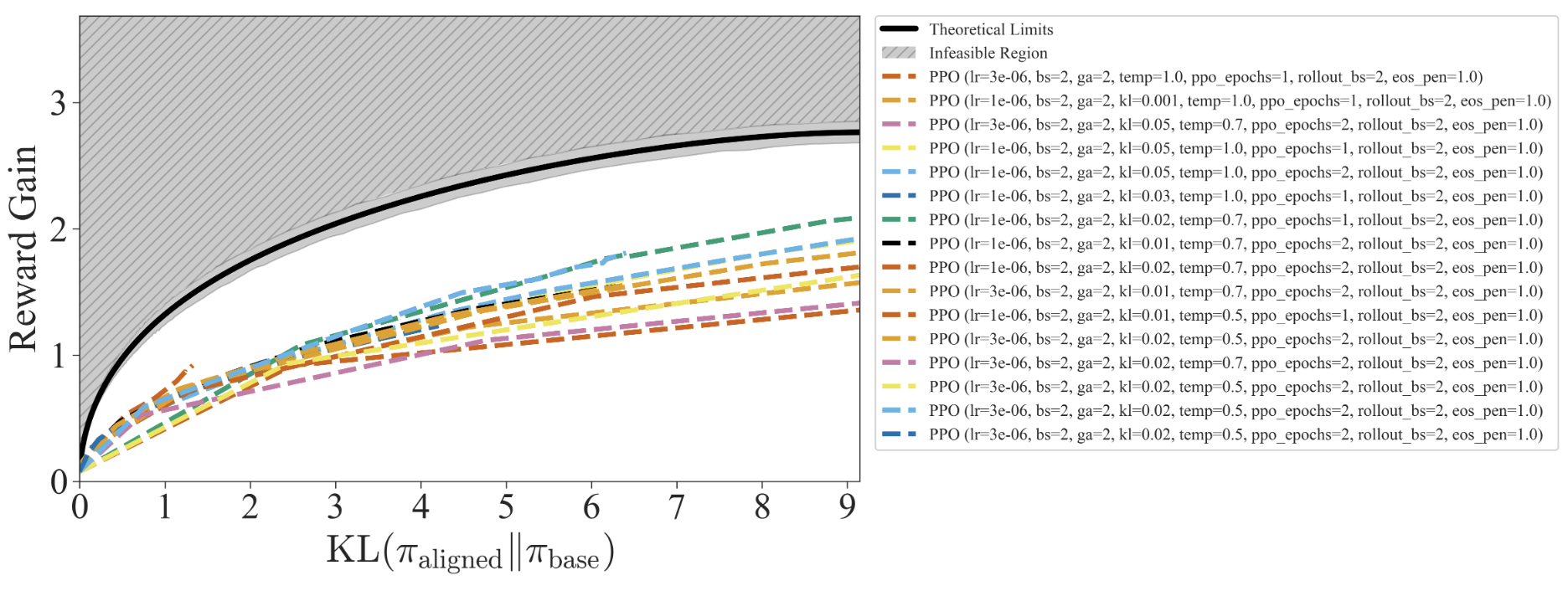}
\caption{\textbf{Hyperparameter sweep for the fine-tuning Zephyr 7B SFT Full using PPO on the Beavertails dataset. }The plot shows the reward gain as the trained deviates from the base model. Abnormalities in the plotted lines are caused by training runs either plateauing before a KL of 9.14 or from oscillating between a small range of KL values.}
    \label{fig:beavertails_ppo_sweep}
\end{figure}

\begin{figure}[htb!]
    \centering
    \includegraphics[width=\linewidth]{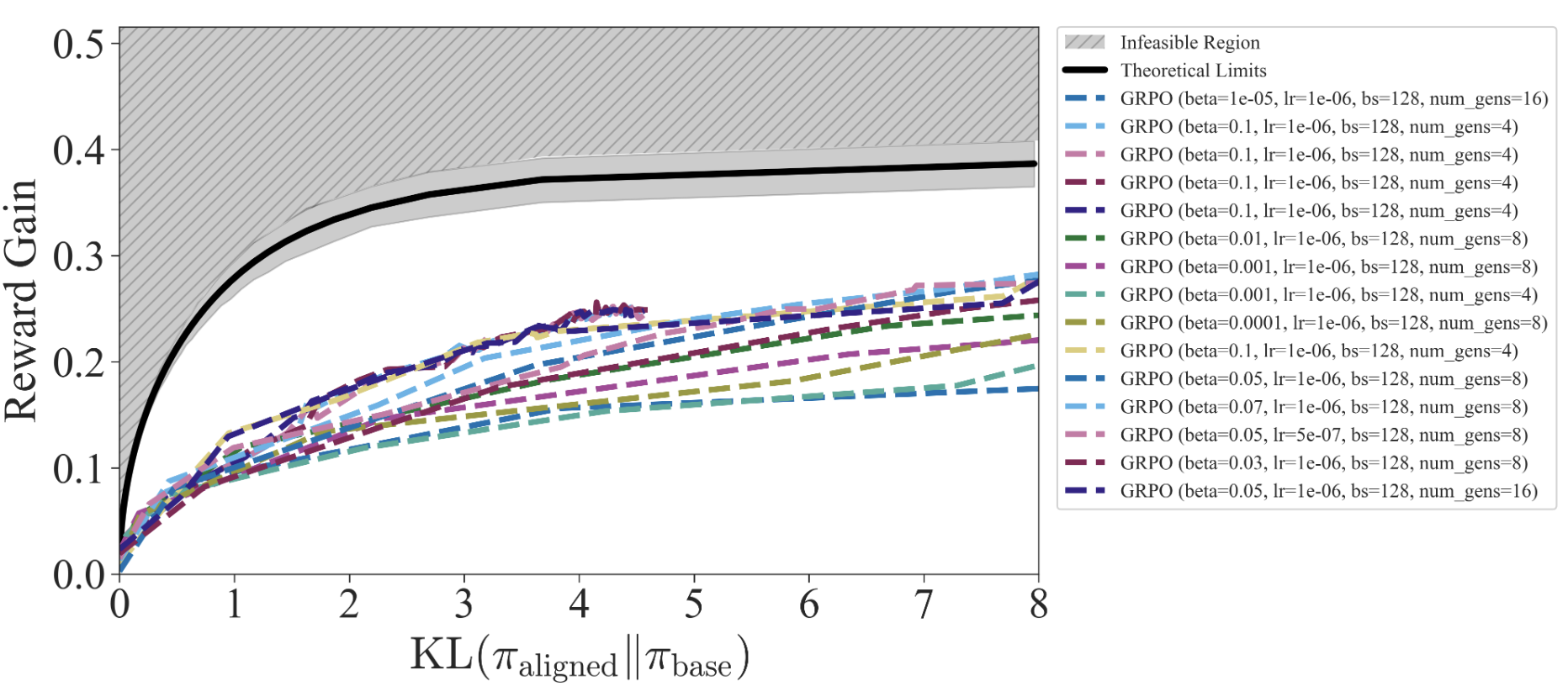}
\caption{\textbf{Hyperparameter sweep for the fine-tuning Pythia 1B SFT using GRPO on the TLDR dataset.} The plot shows the reward gain as the trained deviates from the base model. Abnormalities in the plotted lines are caused by training runs either plateauing before a KL of 7.96 or from oscillating between a small range of KL values.}
    \label{fig:tldr_grpo_sweep}
\end{figure}

\begin{figure}[htb!]
    \centering
    \includegraphics[width=\linewidth]{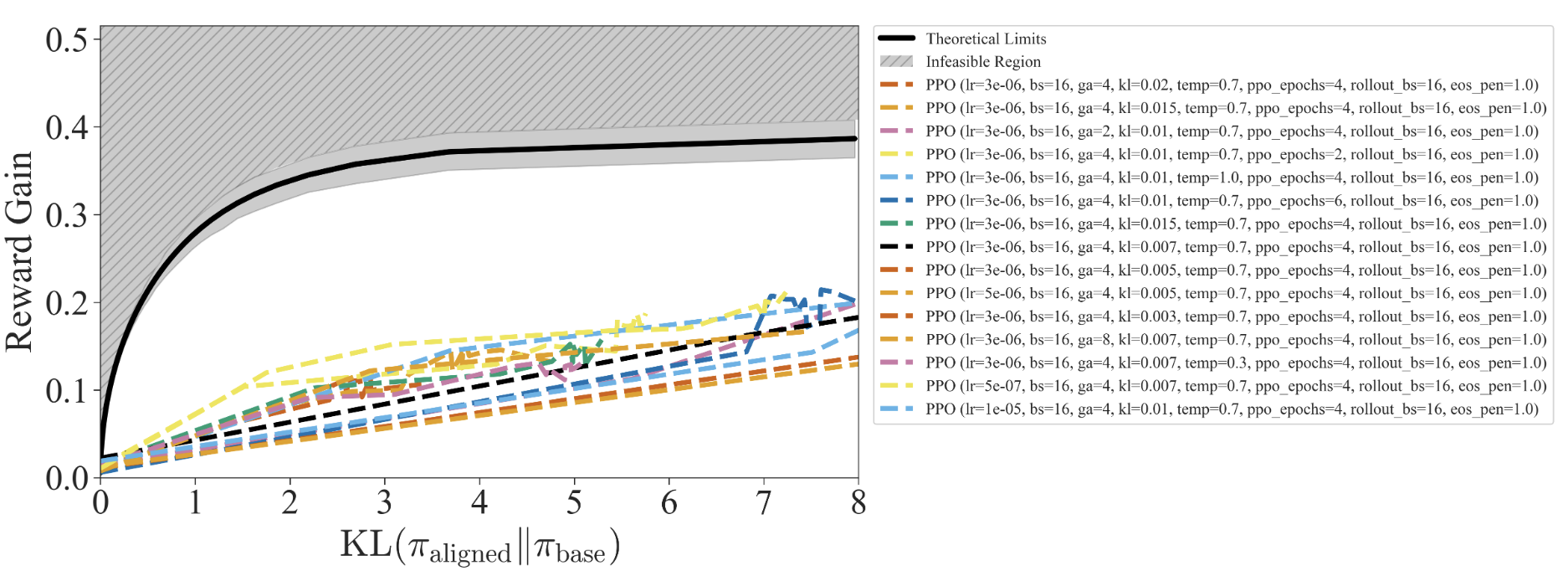}
\caption{\textbf{Hyperparameter sweep for the fine-tuning Pythia 1B SFT using PPO on the TLDR dataset. }The plot shows the reward gain as the trained deviates from the base model. Abnormalities in the plotted lines are caused by training runs either plateauing before a KL of 7.96 or from oscillating between a small range of KL values.}
    \label{fig:tldr_ppo_sweep}
\end{figure}

\section{Licenses of Used Assets}
\label{apx:assets}

The existing assets used in this paper are listed in Table~\ref{tab:existing_assets}. 
The assets include publicly available datasets, model checkpoints, and software libraries used for the theoretical validation and empirical evaluation. 
We cite the original creators of each asset and report the corresponding licenses.

\begin{table}[htb!]
\centering
\caption{Existing assets used in this work.}
\label{tab:existing_assets}
\begin{tabular}{lll}
\toprule
\textbf{Asset} & \textbf{Type} & \textbf{License / Terms} \\
\midrule
Reddit TL;DR \cite{trl-lib-tldr,summarizeStiennon} & Dataset & Modified MIT License \\
BeaverTails \cite{ji2023beavertailsimprovedsafetyalignment} & Dataset & CC BY-NC 4.0 \\
PKU-SafeRLHF \cite{ji2025pkusaferlhfmultilevelsafetyalignment} & Dataset & CC BY-NC 4.0 \\
Pythia 1B  \cite{trl_pythia1b_tldr_sft,biderman2023pythiasuiteanalyzinglarge}& Model & Apache 2.0 \\
Zephyr-7B \cite{tunstall2023zephyrdirectdistillationlm} & Model checkpoint & Apache 2.0 \\
Mistral-7B-v0.1\cite{jiang2023mistral7b} & Model & Apache 2.0 \\
TRL \cite{vonwerra2020trl} & Code library & Apache 2.0 \\
\bottomrule
\end{tabular}
\end{table}

\end{document}